\documentclass[10pt]{article}

\usepackage[margin=1in]{geometry}
\usepackage{natbib}
\usepackage{parskip}

\usepackage{algorithm, algorithmicx, algpseudocode}
\usepackage{amssymb}
\usepackage{amsmath}
\usepackage{graphicx}
\usepackage{mathtools}
\usepackage{needspace} 
\usepackage{multicol}
\usepackage{listings}
\usepackage{amsthm}
\usepackage[shortlabels]{enumitem}
\usepackage[italicdiff]{physics}
\usepackage{cancel}
\usepackage[dvipsnames]{xcolor}
\usepackage{verbatim}
\usepackage{comment}
\usepackage{array}
\usepackage{tikz-cd}
\usepackage{import}
\usepackage{booktabs}
\usepackage{comment}
\usepackage{array}
\usepackage[colorlinks, linkcolor=blue, citecolor=ForestGreen]{hyperref}
\usepackage{forest}
\usepackage{float}
\usepackage{xspace}
\usepackage{aliascnt}
\usepackage{wrapfig}
\usepackage{thm-restate}
\usepackage{adjustbox}
\usepackage{bm}
\usepackage{bbm}
\usepackage{complexity}
\usepackage{makecell}
\usepackage{nicefrac}
\usetikzlibrary{calc, positioning, external, arrows.meta}

\allowdisplaybreaks

\DeclareMathOperator*{\argmax}{argmax}

\let\cite\citep

\newfloat{protocol}{tbp}{lop}
\floatname{protocol}{Protocol}

\newtheorem{theorem}{Theorem}
\numberwithin{theorem}{section}
\newtheorem*{theorem*}{Theorem}

\newtheorem{assumption}[theorem]{Assumption}

\newtheorem{remark}[theorem]{Remark}

\theoremstyle{definition}
\newtheorem{definition}[theorem]{Definition}

\def\vg{{\bm{g}}}

\def\vr{{\bm{r}}}

\def\vx{{\bm{x}}}

\usetikzlibrary{calc, positioning, external, arrows.meta}
\usepackage[most]{tcolorbox}
\colorlet{lightgreen}{green!20}
\tcbset{on line, 
        boxsep=1pt, left=0pt,right=0pt,top=0pt,bottom=0pt,
        colframe=white,colback=lightgreen,  
        highlight math style={enhanced}
        }

\usepackage[short, nocomma]{optidef} 

\usepackage{tcolorbox}

\allowdisplaybreaks

\definecolor{lightblue}{rgb}{0.8, 0.9, 1}
\definecolor{lightred}{rgb}{1, 0.8, 0.8}
\definecolor{lightgreen}{rgb}{0.8, 1, 0.8}

\let\cite\citep

\newfloat{protocol}{tbp}{lop}
\floatname{protocol}{Protocol}

\newcommand{\zetavec}{\boldsymbol{\zeta}}

\def\vg{{\bm{g}}}

\def\vr{{\bm{r}}}

\def\vx{{\bm{x}}}

\newcommand{\circled}[1]{\tikz[baseline=(char.base)]{
    \node[shape=circle,draw,inner sep=1pt] (char) {#1};}}

\newcommand{\ReprMean}{\texttt{RepMean}}
\newcommand{\XOptSafe}{\underline{\mathcal{X}}}
\newcommand{\XPessSafe}{\overline{\mathcal{X}}}

\newcommand{\lambdaMin}{\lambda_{\textnormal{min}}}

\newcommand{\algname}{\textnormal{\ttfamily\hyphenchar\font=`\-DEBORA}\xspace}
\newcommand{\algnamesoft}{\textnormal{\ttfamily\hyphenchar\font=`\-DEBORA-S}\xspace}
\newcommand{\algnamehard}{\textnormal{\ttfamily\hyphenchar\font=`\-DEBORA-H}\xspace}

\usepackage{caption}
\makeatletter
\renewcommand\thanks[1]{%
  \footnotemark%
  \protected@xdef\@thanks{\@thanks
    \protect\footnotetext[\the\c@footnote]{#1}}%
}
\makeatother

\title{Replicable Constrained Bandits}
\author{
Matteo Bollini\thanks{Equal contribution.} \\
    \texttt{matteo.bollini@polimi.it} \\
    Politecnico di Milano 
\and
Gianmarco Genalti\footnotemark[1] \\
    \texttt{gianmarco.genalti@polimi.it} \\
    Politecnico di Milano 
\and
    Francesco Emanuele Stradi\footnotemark[1] \\
    \texttt{francescoemanuele.stradi@polimi.it} \\
    Politecnico di Milano 
\and
    Matteo Castiglioni \\
    \texttt{matteo.castiglioni@polimi.it} \\
    Politecnico di Milano
\and
Alberto Marchesi \\
    \texttt{alberto.marchesi@polimi.it} \\
    Politecnico di Milano 
\and
}
\date{\today}

\begin{document}

\maketitle

\begin{abstract}
 Algorithmic \emph{replicability} has recently been introduced to address the need for reproducible experiments in machine learning.
 A \emph{replicable online learning} algorithm is one that takes the same sequence of decisions across different executions in the same environment, with high probability. We initiate the study of algorithmic replicability in \emph{constrained} MAB problems, where a learner interacts with an unknown stochastic environment for $T$ rounds, seeking not only to maximize reward but also to satisfy multiple constraints. Our main result is that replicability can be achieved in constrained MABs. 
 Specifically, we design replicable algorithms whose regret and constraint violation match those of non-replicable ones in terms of $T$.
 As a key step toward these guarantees, we develop the first replicable UCB-like algorithm for \emph{unconstrained} MABs, showing that algorithms that employ the optimism in-the-face-of-uncertainty principle can be replicable, a result that we believe is of independent interest.
\end{abstract}

\section{Introduction}
\label{sec:introduction}



To produce scientific findings that are valid and reliable, the reproducibility of the experimental process is vital. However, in a 2016 survey involving 1{,}500 researchers, it was reported that more than 70\% of respondents had attempted but failed to replicate another researcher’s findings~\citep{baker20161}. A more recent article~\citep{ball2023ai} discussed how recent trends in AI research are creating a replicability crisis, which may have disruptive consequences in several scientific fields that are increasingly relying on AI tools. Indeed, major AI conferences have already taken some countermeasures to respond to this crisis (see, \emph{e.g.}, the NeurIPS 2019 Reproducibility Program~\citep{pineau2021improving}).

The concept of \emph{algorithmic replicability} in machine learning has recently been introduced to address reproducibility issues within a principled framework (see, \emph{e.g.},~\citep{impagliazzo2022reproducibility,ghazi2021user,ahn2022reproducibility}). Reproducibility in \emph{online learning}, which is the focus of this paper, has been addressed by only a few works~\citep{esfandiari2022replicable,ahmadi2024replicable,komiyama2024replicability}. 
A \emph{replicable online learning algorithm} is one that takes the \emph{same} sequence of decisions over the rounds across different executions in the same environment, with high probability.
Previous works have only considered replicability in the classic \emph{multi-armed bandit} (MAB) problem, where the learner faces a sequence of decisions over $T$ rounds. At each round $t$, the learner chooses an action $a_t$ from a finite set of $K$ available actions and then receives a reward $r_t(a_t)$ for the selected action. The learner's goal is to maximize the cumulative reward or, equivalently, to minimize the \emph{regret} with respect to always playing an optimal action in hindsight.
Perhaps surprisingly, previous results~\citep{esfandiari2022replicable,ahmadi2024replicable} showed that, in stochastic MABs, not only do replicable algorithms exist, but they also achieve regret that is almost the same as that of non-replicable algorithms in terms of $T$.

A recent and growing line of research in online learning is concerned with \emph{constraints} (see, \emph{e.g.},~\citep{Mannor, castiglioni2022online,pacchiano,bounded}). In a \emph{constrained} MAB, the learner is subject to $m$ cost constraints and, at each round $t$, also observes a cost $g_{t,i}(a_t)$ associated with each constraint $i$. The learner's goal thus becomes twofold: minimizing the regret while simultaneously ensuring that the constraints are satisfied.
Constrained MABs find application in several real-world scenarios of interest, where it is often the case that additional requirements need to be satisfied during learning. For instance, this is the case in online auctions~\citep{badanidiyuru2018bandits}---where a bidder must \emph{not} deplete the budget---wireless communication~\citep{Mannor}, and safety-critical systems~\citep{amodei2016concrete}.
Despite this, replicability in constrained MAB problems is still unexplored, to the best of our knowledge.

In this paper, we initiate the study of \emph{algorithmic replicability in stochastic constrained} MABs addressing the following question,
    \emph{``are there replicable algorithms that (almost) match the guarantees achieved by non-replicable ones?''}
%
We answer the question affirmatively, as described next.

\subsection{Original Contributions}
The main challenge in designing replicable algorithms for stochastic constrained MABs is that the state-of-the-art replicable algorithms for unconstrained settings~\citep{esfandiari2022replicable,komiyama2024replicability} cannot be easily adapted to handle constraints. This is because such algorithms implement complex action-elimination procedures. 
On the one hand, action-elimination approaches do not readily generalize to settings in which the learner must optimize over arbitrary polytopes contained in the simplex, and hence must select randomized strategies (rather than deterministic actions) throughout the learning process. This is the case of constrained MABs, where the safety constraints restrict the decision space to a generic polytope and the optimal solution is generally a non-uniform \emph{distribution} over actions.
On the other hand, those procedures generally fail to cope with the tension among multiple objectives that is inherently present in constrained settings. 
Intuitively, since constraint estimation can potentially classify high-reward action-selection strategies as infeasible, reward-based action elimination could prematurely discard optimal strategies. Even if such strategies were reintroduced later, they would still be underexplored, thus invalidating the standard analysis of arm-elimination algorithms, which relies on the uniform exploration of the active decision space.

As a first step toward answering our research question, we design a novel replicable algorithm for learning in stochastic (unconstrained) MABs. This is the first of its kind to be based on a UCB-like approach, which is crucial for extending the algorithm to the constrained case. The algorithm, called \algname{}, is based on a \emph{dynamic epoch-based} approach. At the beginning of each epoch $h$, the algorithm selects an action using a UCB-style rule and plays it for all the rounds of the epoch. The epoch ends when the selected action has doubled the number of rounds in which it has been chosen. Crucially, the UCB-style action selection is performed using suitable \emph{replicable estimators} and their confidence bounds~\citep{impagliazzo2022reproducibility}. The dynamic nature of \algname{} allows it to go beyond classic epoch-based approaches in which epochs are \emph{fixed a priori}, thereby avoiding the use of an arm-elimination approach.
Remarkably, the regret guarantees of \algname{} match, up to constant factors, those of state-of-the-art replicable algorithms~\citep{esfandiari2022replicable,komiyama2024replicability}, without resorting to action elimination. This may be of interest not only for constrained problems, but also for all settings in which a UCB-like approach is desirable over action elimination (such as, \emph{e.g.}, non-stationary settings).

%

Our results for constrained MABs are concerned with two classic settings, called the \emph{soft} and \emph{hard} constraints settings (see, \emph{e.g.},~\cite{gangrade2024safe,pacchiano}).

In the \emph{soft} constraints setting, the goal is to minimize the regret while controlling the \emph{cumulative positive violation} of the constraints. For this setting, we propose an extension of \algname{}, called \algnamesoft{}, which works similarly but, at each epoch, employs an action-selection (possibly randomized) strategy chosen via a UCB-style approach from a subset of strategies that \emph{optimistically} satisfy the constraints. Such a subset is computed using replicable estimators of the constraint costs. We prove that \algnamesoft{} attains both regret and constraint violation that grow as $\widetilde{\mathcal{O}}(\sqrt{T})$ in $T$.\footnote{ The $\widetilde{\mathcal{O}}(\cdot)$ notation hides poly-logarithmic factors.}

In the \emph{hard} constraints setting, the goal is to minimize the regret while guaranteeing that the constraints are satisfied at every round. The algorithm for this setting, called \algnamehard{}, builds on top of \algnamesoft{} by only assuming access to a suitable \emph{strictly} constraint-satisfying strategy $\vx^\diamond$, as is standard in the literature (see, \emph{e.g.},~\citep{pacchiano}). The approach of \algnamesoft{} may fail in the hard constraints setting, since the strategy selected by the algorithm during an epoch may violate some constraints throughout all the rounds of the epoch. \algnamehard{} resolves this issue by mixing such a strategy with $\vx^\diamond$, in order to balance constraint satisfaction and optimism. The regret of \algnamehard{} grows as $\widetilde{\mathcal{O}}(\sqrt{T})$ in $T$, while satisfying the constraints at every round with high probability.

In conclusion, in both the \emph{soft} and \emph{hard} constraints settings, the regret guarantees of our algorithms match, in terms of $T$, those of state-of-the-art non-replicable algorithms~\citep{gangrade2024safe,pacchiano}. This shows that replicability comes at almost no cost in constrained MAB problems, as is the case for unconstrained ones.

\subsection{Related Works}
Recent works have addressed the problem of developing replicable regret minimization strategies in \emph{unconstrained} settings, such as, in the standard MAB problems. We summarize them in the following.

%
\citep{esfandiari2022replicable} is the first work to formally introduce the notion of \textit{replicable bandit algorithm}. There, the authors focus on three stochastic MAB problems: the standard one, linear bandits, and bandits with an infinite action set. For the stochastic MAB problem, they introduce two algorithms which rely on arm elimination strategies. 
While the first algorithm is simpler and more intuitive, the second attains better theoretical guarantees. Their best algorithm is guaranteed to suffer a cumulative regret of at most $\mathcal{O}(\frac{K^2}{\rho^2}\sum_{a \neq a^*} \frac{\ln T}{\Delta(a)})$,  where $\Delta(a)$ is the sub-optimality of action $a$ w.r.t. to the optimal one $a^*$. This bound can be easily converted to a worst-case form of $\widetilde{\mathcal{O}}(\rho^{-1}K^\frac{3}{2}\sqrt{T})$.
Subsequently, \citet{komiyama2024replicability} provide two new algorithms to improve the instance-dependent guarantees in both the multi-armed and the linear bandit problem. In particular, the authors provide an algorithm based on arm elimination called \texttt{RSE}, which is guaranteed to suffer a cumulative regret of at most $\mathcal{O}(\sum_{a \neq a^*} \frac{\ln T}{\Delta(a)} + \frac{K^2\ln \ln T}{\rho^2 \Delta(a)}  )$, improving the previous result from \cite{esfandiari2022replicable} by decoupling the standard MAB regret from an additional contribution due to the replicability, which is of lower order in $T$.
\cite{ahmadi2024replicable} generalizes the replicable bandit problem by allowing for arbitrary sequences of reward distributions. In fact, this setting expands the concept of replicability from stochastic bandits to their adversarial counterparts. They establish a regret lower bound of $\Omega(\rho^{-1}\sqrt{T})$ for both settings, but only provide algorithmic contributions for the \textit{experts} setting.

Two additional, closely related lines of work are the literature on constrained MABs~\citep{pacchiano,Unifying_Framework} and on constrained MDPs~\citep{Altman1999ConstrainedMD,Exploration_Exploitation}. In both these frameworks, two kinds of results are generally presented. On the one hand, some works focus on guaranteeing per-round satisfaction of the constraints, by assuming the knowldege of a strictly safe strategy $\vx^\diamond$~\citep{pacchiano,bounded}. At the same time, those works attains sublinear regret of order $\widetilde{\mathcal{O}}(\nicefrac{1}{\lambda}\sqrt{T})$, where $\lambda$ measure strictly safety of $\vx^\diamond$, which is tight~\citep{stradi2024learning}. On the other hand, many works focus on guaranteeing $\widetilde{\mathcal{O}}(\sqrt{T})$ regret allowing the algorithm to violate the constraints sublinearly, that is, attaining $\widetilde{O}(\sqrt{T})$ violation of the constraints~\citep{Exploration_Exploitation, Unifying_Framework}. Notice that, in this setting, the $1/\lambda$ dependence in the regret bound can be removed. 
To the best of our knowledge, this is the first work focusing on both constraints and replicability.

Due to space constraints, we refer to Appendix~\ref{App:related} for the complete discussion on constrained online learning.
\section{Preliminaries}\label{sec:preliminaries}
\label{sec:preliminaries}
In this section, we review the key definitions and fundamental technical tools for constrained MAB problems, as well as the notion of replicability in learning.

\subsection{Constrained MABs}

In a \emph{constrained} MAB, a learner faces a decision problem involving $K \in \mathbb{N}_{>0}$ actions over $T \in \mathbb{N}_{>0}$ rounds.
At each round $t \in [T]$,\footnote{We let $[n] \coloneqq \{1,2,\dots,n\}$ be the first $n$ natural numbers.} the learner selects an action $a_t \in [K]$ (possibly at random) and observes a reward $r_t(a_t) \in [0,1]$.  
Furthermore, the learner also observes a cost $g_{t,i}(a_t) \in [0,1]$ for each constraint $i \in [m]$.
We consider a stochastic setting in which, for each action $a \in [K]$ and each constraint $i \in [m]$, rewards and costs are independently drawn from fixed distributions $\mathcal{R}_a$ and $\mathcal{G}_{a,i}$, respectively.
The means of these distributions are given by the unknown vectors
$\vr$ and $\vg_i$, for all $i \in [m]$.  
We will refer to their $a$-th element as $\vr(a)$ and $\vg_i(a)$, respectively.

At each $t \in [T]$, the learner selects an action $a_t \in [K]$ at random according to a probability distribution $\vx_t \in \Delta_K$,\footnote{We denote by $\Delta_n$ the $(n-1)$-dimensional simplex, namely, we let $\Delta_n \coloneqq \{ \vx \in \mathbb{R}^n_{\geq 0} \mid \mathbf{1}^\top \vx = 1 \}$.} referred to as the learner's \emph{strategy}.
The learner chooses $\vx_t$ by employing a learning algorithm, based on the history of their observations up to round $t-1$.
We say that a constraint $i \in [m]$ is \emph{satisfied} by a learner's strategy $\vx \in \Delta_K$ if its expected cost does \emph{not} exceed a known threshold $\alpha_i \in [0,1]$, \emph{i.e.}, if $\vg_i^\top \vx \le \alpha_i$.
When a strategy satisfies all the $m$ constraints, we say that it is \emph{safe} for the learner.

As is customary in the literature (see, \emph{e.g.},~\citep{Unifying_Framework}), we measure the performance of a learning algorithm over $T$ rounds in terms of the regret with respect to an optimal \emph{safe} strategy $\vx^\star \in \Delta_K$, which is formally defined as follows:
\begin{equation*}
	\arraycolsep=1.4pt
\begin{array}{*3{>{\displaystyle}l}} 
	\vx^\star \in &\argmax_{\vx \in \Delta_K} & \, \vr^\top \vx \quad \text{s.t.} \\
	& &\vg_i^\top \vx \le \alpha_i \quad \forall i \in [m].
\end{array}
\end{equation*}
The \textit{expected cumulative regret} (hereafter simply called \emph{regret} for short) is then formally defined as follows:
\[
R_T \coloneqq \sum_{t\in[T]} \vr^\top \vx^\star - \sum_{t \in [T]} \vr^\top \vx_t.
\]
To evaluate the performance of a learning algorithm in terms of constraint satisfaction, two different settings are typically considered (see, \emph{e.g.},~\citep{gangrade2024safe,pacchiano}), which we describe below.

\paragraph{Soft Constraints} In the \emph{soft constraints} setting, the constraint satisfaction of a learning algorithm is measured by the cumulative \emph{positive} constraint violation over the $T$ rounds (called \emph{violation} for short), which is defined as:
\[
V_T \coloneqq  \max_{i \in [m]} \sum_{t \in [T]} \max\left\{0,\vg_i^\top \vx_t -\alpha_i\right\}.
\]
Notice that a negative constraint violation $\vg_i^\top \vx_t - \alpha_i < 0$ at round $t$ (\emph{i.e.}, a satisfaction of the constraint) is capped at $0$, so that it does \emph{not} cancel out a positive violation incurred in another round.
%
%
In this setting, the goal is to design learning algorithms that achieve both sublinear regret and sublinear violation, \emph{i.e.}, $R_T = o(T)$ and $V_T = o(T)$.

\paragraph{Hard Constraints} In the \emph{hard constraints} setting, the goal is to design learning algorithms 
that satisfy the constraints at every round $t \in [T]$ with high probability, while simultaneously attaining sublinear regret. Clearly, the initial rounds of learning are the most critical, since no information about the environment is available and, accordingly, it is \emph{not} possible to select a safe strategy. For this reason, hard constraints settings can be addressed only under the following 
two (necessary) assumptions, which are common in the literature 
(see, \emph{e.g.},~\citep{pacchiano,bounded}). Intuitively, these assumptions are needed to guarantee a minimum amount of exploration in the first rounds.
\begin{assumption}[Slater's condition]\label{ass:Slater}
    There exists a strictly feasible strategy $ \vx^{\diamond} \in \Delta_K:\vg_i^\top \vx^\diamond<  \alpha_i$ for all $i\in[m]$.
\end{assumption}
%
%
\begin{assumption}\label{ass:knowldege}
    The learner knows the strictly feasible strategy $\vx^\diamond \coloneqq\argmax_{\vx\in\Delta_K}\min_{i \in [m]} \{\alpha_i-\vg_i^\top\vx\}$ and its costs $\lambda_i\coloneqq\vg_{i}^\top\vx^\diamond$.\footnote{Previous works (see, \emph{e.g.},~\citep{pacchiano,bounded,safetree}) usually assume knowledge of a generic strictly feasible strategy. For ease of presentation, in this work we assume knowledge of a strategy that satisfies the constraints as much as possible. Nonetheless, our results can be generalized to the case commonly considered in the literature.}
    We also let $\lambdaMin\coloneqq \min_{i\in[m]} \left\{\alpha_i-\lambda_i\right\}$ be the margin by which $\vx^\diamond$ satisfies the constraints.
\end{assumption}

\subsection{Replicability in Learning}

Our main goal in this paper is to develop \emph{replicable} learning algorithms for constrained MABs. Intuitively, a replicable algorithm should suggest the same sequence of decisions across different realizations of the stochastic events when interacting with the same environment. In particular, a replicable learning algorithm takes as input a source of internal randomness $\xi$ and, at each $t \in [T]$, outputs a strategy $\vx_t \in \Delta_K$ based on the history of the learner's observations up to round $t-1$---including selected actions and collected rewards and costs. Moreover, the selected action $a_t \sim \vx_t$ is sampled using the same internal randomness $\xi$. We adopt the definition of replicability of~\citet{esfandiari2022replicable}.
%
%
\begin{definition}[$\rho$-Replicable algorithm]
	Fix $\rho \in (0,1)$. We say that a learning algorithm is \emph{$\rho$-replicable} if, for any two runs of the algorithm on the same constrained MAB instance, when given as input the same source of internal randomness $\xi$, the following condition holds:
	\begin{equation*}
		\mathbb{P}_\xi\left( (\vx^1_1,\vx^1_2,\dots, \vx^1_T) = (\vx^2_1,\vx^2_2,\dots, \vx^2_T)  \right) \ge 1-\rho,
	\end{equation*}
	where $\vx^i_t \in \Delta_K$ denotes the strategy selected by the algorithm at round $t \in [T]$ during the $i$-th execution.
\end{definition}

\begin{algorithm}[!htp]\caption{Replicable Mean Estimator (\texttt{RepMean})}
    \begin{algorithmic}[1]
    \Require $\mathcal{S} \in [0,1]^n$, $0 < 2\delta < \rho < 1$, $\xi$
    \State $\alpha \gets \frac{2\delta}{1+\rho-2\delta}$ \label{line:alpha}
    \State $\alpha_{\text{off}} \sim \mathcal{U}_y\left([0,\alpha]\right)$ \label{line:alpha_off}
    \State $R \gets \{[0, \alpha_{\text{off}} ), [ \alpha_{\text{off}}, \alpha_{\text{off}}+\alpha),\ldots,[\alpha_{\text{off}}+k \alpha , 1) \}$ \label{line:R}
    \State $\bar{r} \gets \frac{1}{n}\sum_{i = 1}^n \mathcal{D}_i$ \label{line:bar_r}
    \State $z \gets \lfloor \max\{\frac{\bar{r}-\alpha_{\text{off}}}{\alpha},0\}\rfloor$ \label{line:z}
    \State $\widehat{r} \gets \min\{\alpha_{\text{off}} +\left(z+\frac{1}{2}\right)\alpha, 1\}$ \label{line:hat_r}
    \State \label{line:return} \textbf{Return} $\widehat{r}$ 
    \end{algorithmic}
    \label{alg:rep_mean}
\end{algorithm}

\begin{remark} 
By fixing the algorithm's internal randomness, we ensure that also the actions selected across two independent runs coincide with high probability.  In particular, the two actions $a_t^1 = a_t^2$ when $a_t^1 \sim \vx_t^1$ and $a_t^2 \sim \vx_t^2$.
Let us also remark that fixing the internal randomness of the algorithm is standard in the literature on replicable learning~\citep{impagliazzo2022reproducibility}.
\end{remark}

\subsubsection{Replicable Mean Estimation}

In Algorithm~\ref{alg:rep_mean}, we provide the pseudocode of \texttt{RepMean}, namely, the algorithm originally proposed by \citet{impagliazzo2022reproducibility} to compute a $\rho$-replicable mean estimator. \texttt{RepMean} requires as input a set $\mathcal{S}$ of $n$ samples drawn from a random variable with support in $[0,1]$, a tolerance $\delta \in (0, 1)$, a replicability parameter $\rho \in (2\delta,1)$, and a source of randomness $\xi$. The algorithm works by discretizing the interval $[0,1]$ into a finite set of possible outputs. To do so, it computes the size $\alpha$ of every interval (Line~\ref{line:alpha}), a randomized offset $\alpha_{\text{off}}$ (Line~\ref{line:alpha_off}), and the resulting grid $R$ (Line~\ref{line:R}). Then, \texttt{RepMean} computes the sample average $\bar{r}$ given $\mathcal{S}$ (Line~\ref{line:bar_r}) and identifies the index of the cell of the grid where it belongs (Line~\ref{line:z}). Finally, the algorithm returns the middle point of the same cell as output (Line~\ref{line:hat_r}).

We remark that $\alpha_{\text{off}}$ is chosen (at random) according to the internal randomness $\xi$, which ensures replicability.
Formally, the following holds.

\begin{restatable}{lemma}{restReplEst}[\citep[Theorem 2.3]{impagliazzo2022reproducibility}]
    \label{lem:replicable_estiamtor}
	Let $\mathcal{D}$ be a distribution with support in $[0,1]$ and mean $r$, and let $\epsilon, \rho', \delta' \in [0,1]$ with $\rho' > 2\delta'$.
	Then, there exists a $\rho'$-replicable algorithm that, given a set $\mathcal{S}$ of $\frac{2\log(\nicefrac{2}{\delta'})}{\epsilon^2(\rho'-2\delta')^2}$ i.i.d. samples drawn from $\mathcal{D}$, outputs an estimator $\widehat{r}$ such that~$\mathbb{P}(|\widehat{r} -r| \le \epsilon) \ge 1 -\delta'$.
\end{restatable}
Lemma~\ref{lem:replicable_estiamtor} implies that, given $n \in \mathbb{N}_{>0}$ i.i.d.~samples from a distribution $\mathcal{D}$ with mean $r$, it is possible to design a $\rho$-replicable estimator $\widehat{r}$ of $r$ with precision $\epsilon = \sqrt{\frac{2\log(\nicefrac{2}{\delta'})}{n(\rho'-2\delta')^2}}$, where $\delta'<\rho'/2$ is the error probability.
\section{Warm-Up: Unconstrained MABs}
\label{sec:unconstrained}

In this section, we introduce \emph{Dynamic Epoch-Based Optimism with ReplicAbility} (\algname),
%
%
a $\rho$-replicable algorithm for unconstrained MABs. The goal of this section is twofold. First, it serves as an introduction to the general algorithmic template that allows us 
to deal with constrained MABs while being $\rho$-replicable. Second, it shows how \algname itself attains regret of the same order as state-of-the-art approaches for unconstrained MABs, while simultaneously avoiding arm-elimination. 

\subsection{Beyond Arm Elimination}
Our algorithm (\algname) leverages the well-known \textit{optimism} principle~\cite{auer2002finite}. Optimistic algorithms are the gold standard for most stochastic MAB problems. Indeed, as discussed in~\cite{esfandiari2022replicable}, a straightforward approach to obtain a $\rho$-replicable no-regret algorithm for stochastic MABs is to plug the replicable estimator from~\cite{impagliazzo2022reproducibility} into the \texttt{UCB1} algorithm of~\cite{auer2002finite} and adjust the confidence intervals accordingly. However, requiring all confidence intervals to be valid at every round, together with the estimators being $\rho$-replicable, leads to a statistical complexity of $K^2 T^2$, which is excessive for any learning algorithm.

The issue described above has been addressed by reducing the number of statistical estimations and arm comparisons to a controlled amount~\cite{esfandiari2022replicable,komiyama2024replicability}. For instance, this can be achieved by partitioning the $T$ rounds into epochs and computing statistics only between consecutive epochs, or by performing an initial exploration phase and then committing to a single action. Since the number of estimations is substantially smaller than the total number of rounds (\emph{e.g.}, logarithmic in $T$), the statistical complexity is therefore controlled.

Existing algorithms all share another common feature: they rely on \textit{eliminating} actions, \emph{i.e.}, comparisons between actions are used to permanently discard certain arms from the sampling pool. This is a direct consequence of adopting the aforementioned epoch-based protocols, which split the learning process into \emph{fixed} phases to guarantee replicability, thus making action elimination approaches the most natural choice. 
Nonetheless, \emph{fixed} phases presents various drawbacks. For instance, they do not work when distributions over actions are selected by the algorithm (as in the case of constrained MABs), since they do not take into account the probability associated to each action. Moreover, since phase lengths increase over time but are \emph{fixed a priori}, a naive UCB-style strategy---which keeps exploring throughout the entire learning process---may fail to guarantee sublinear regret if a suboptimal action is selected at the start of the last (and longest) phases.

Unfortunately, even though action elimination algorithms can attain nearly optimal regret in standard MABs---and can be easily extended to the same \emph{dynamic} epoch-based protocol we describe later---, they become nonviable in more general settings, such as constrained MABs 
and nonstationary MABs, where an action discarded early can later turn out to be optimal. For instance, in the constrained MABs case, action elimination procedures generally fail to cope with the tension among multiple objectives. Indeed, constraints estimation may label high-reward strategies as \emph{non-safe}, whereas reward based elimination may prematurely discard strategies which are truly optimal compared to the other \emph{safe} strategies. Reintroducing such strategies at a later stage does not resolve the issue, since they remain insufficiently explored, thereby breaking the standard action elimination analysis that hinges on (approximately) uniform exploration among active actions. Additionally, action elimination approaches do not easily extend to settings where \emph{strategies} are selected in place of actions, as in the case of constrained MABs where the per-round optimization must performed by the algorithm arbitrary polytopes contained in the simplex.

Crucially, \algname is the first replicable algorithm to achieve nearly optimal regret guarantees while avoiding any arm elimination procedure. Indeed, we show how a \emph{dynamic} epoch-based protocol allows one to employ a UCB-style approach while still guaranteeing replicability.

\subsection{Algorithm}

\begin{algorithm}[!htp]\caption{\algname}
    \begin{algorithmic}[1]
    \Require $T \in  \mathbb{N}_{>0}$, $K > 1$, $0 < 2\delta < \rho < 1, \xi$.
    \State $\delta' \gets \frac{\delta}{2K \lceil\log(T)\rceil}$,
    $\rho' \gets \frac{\rho}{2K \lceil\log(T)\rceil}$ \label{line:unconstrained_init_1}
    \State $h \gets 0$
    \State $N_1(a) \gets 0, \; \widetilde{N}_0(a) \gets 0, \; \widehat{\vr}_0(a) \gets \frac{1}{2} \quad \forall a \in [K]$ 
    \State $a^\circ_0 \gets \argmax_{a \in [K]} \widehat{\vr}_0(a) +\zetavec_0(a)$ \label{line:unconstrained_init_2}
    \For{$t=1,\dots,T$}
        \If{$N_t(a^\circ_h) \ge \max\{2 \widetilde{N}_h(a^\circ_h),1\}$} \label{line:unconstrained_confront}
            \State $h \gets h+1$ \label{line:unconstrained_epoch_end}
            \State $\widetilde{N}_h(a) \gets N_t(a) \quad \forall a \in [K]$ \label{line:unconstrained_ucb_1}
            \State $\mathcal{S}_{a^\circ_{h-1}}(t) \gets \{r_{\tau}(a)\mid \tau < t, a_{\tau}=a\}$
            \State $\widehat{\vr}_h(a) \gets 
            \begin{cases}
                \ReprMean( \mathcal{S}_{a}(t), \delta', \rho', \xi ) &a=a^\circ_{h-1} \\
                \widehat{\vr}_{h-1}(a) & a \neq a^\circ_{h-1}
            \end{cases}
            $
            
            \State $a^\circ_h \gets \argmax_{a \in [K]} \widehat{\vr}_h(a) +\zetavec_h(a)$ \label{line:unconstrained_ucb_2}
        \EndIf
        \State Select $a_t = a^\circ_h$ 
        \label{line:unconstrained_choice_1}
        \State Receive $r_t(a_t)$
        \State $N_{t+1}(a) \gets N_t(a) + \mathbbm{1}\{a=a_t\} \quad \forall a \in [K]$ 
        \label{line:unconstrained_choice_2}
    \EndFor
    \end{algorithmic}
    \label{alg:main_unconstrained}
\end{algorithm}

In Algorithm~\ref{alg:main_unconstrained}, we report the pseudocode of \algname{}. 
At a high level, the algorithm performs a UCB-based action selection at the beginning of each epoch and continues to choose the selected action throughout the entire epoch. Unlike existing approaches for replicable MABs, the epoch size is \emph{not} fixed \emph{a priori}. Instead,  it is dynamically chosen given the history of both observations and decisions. This approach allows \algname{} to operate without eliminating actions or committing to a single one.

Algorithm~\ref{alg:main_unconstrained} starts by initializing the necessary quantities for epoch-based and UCB-style algorithms (Lines~\ref{line:unconstrained_init_1}--\ref{line:unconstrained_init_2}). In particular, at round $t \in [T]$, it keeps track of both $N_t(a) \coloneqq \sum_{\tau = 1}^{t-1} \mathbbm{1}\{a_{\tau} = a\}$, \emph{i.e.}, the number of times action $a \in [K]$ has been selected up to (but excluding) the current round $t$, and $\widetilde{N}_h(a) \coloneqq N_{t_h}(a)$, \emph{i.e.}, the number of times action $a \in [K]$ has been chosen up to round $t_h$, which is the starting round of the current epoch $h$. At the beginning of each round, these two quantities are compared (Line~\ref{line:unconstrained_confront}), and if the currently selected action has doubled its counter $N_t(a)$ since the beginning of the current epoch, the epoch terminates (Line~\ref{line:unconstrained_epoch_end}) and a UCB-based selection of the action for the next epoch is performed (Lines~\ref{line:unconstrained_ucb_1}--\ref{line:unconstrained_ucb_2}). Note that \algname uses its own internal source of randomness $\xi$, provided in input, to compute the replicable mean estimator. Specifically, at the beginning of epoch $h$, \algname computes a $\rho'$-replicable mean estimator $\widehat{\boldsymbol{r}}_{h}(a_{h-1}^\circ)$ for the action $a_{h-1}^\circ$ selected in the previous epoch, together with the corresponding confidence interval width defined as:
\begin{equation}
    \label{eq:zetavec_def}
    \zetavec_{h}(a) \coloneqq \sqrt{\frac{2\ln(\nicefrac{2}{\delta'})}{\max\{\widetilde{N}_h(a),1\}(\rho'-2\delta')^2}}.
\end{equation}
For all $a \neq a_{h-1}^\circ$, the algorithm keeps the same estimator and confidence interval width as in the previous epochs, since no additional samples have arrived in the meantime.

Throughout an epoch $h$, the selected action is continuously chosen until the epoch-termination condition is met (Lines~\ref{line:unconstrained_choice_1}--\ref{line:unconstrained_choice_2}). Notice that the decision of the algorithm is deterministic when conditioned on the past history, as no randomization occurs; thus, $\boldsymbol{x}_t = \boldsymbol{e}_{a_h^\circ}$ during the epoch $h$, where $\boldsymbol{e}_{a_h^\circ}$ is the $a_h^\circ$-th element of the canonical basis.

The following lemma formally states the two key properties of \algname: an upper bound on the number of epochs, which we denote with $H$, and its $\rho$-replicability.
\begin{restatable}{lemma}{ReplicableUnconstrainedLemma}
    \label{lem:replicability_and_epochs_unconstrained}
    The number of epochs $H$ of \algname is at most $K\lceil\log_2(T)\rceil$.
    Furthermore, \algname is $\rho$-replicable. 
\end{restatable}
The two properties are linked, as replicability is established by applying a union bound over the number of epochs. The complete proof can be found in Appendix~\ref{apx:unconstrained}.

Now, we bound the expected regret of \algname.
\begin{restatable}{theorem}{RegretUnconstrained}
    Let $\delta \in (0,1)$ and $\rho \in (2\delta, 1)$. Consider an instance of the stochastic $K$-armed MAB problem defined by the reward gaps $\Delta(a)\coloneqq \arg\max_{\bar a\in[K]}\vr(\bar a)-\vr(a)$ for all $a\in[K]$. Then, with probability at least $1-\delta$, \algname suffers an expected cumulative regret bounded as:
    \begin{equation*}
        R_T \le \widetilde{\mathcal{O}}\left(\frac{1}{(\rho-\delta)^2} \log\left(\frac{1}{\delta}\right)  K^2\log^2(T)\sum_{a \neq a^\star}\frac{1}{\Delta(a)} \right).
    \end{equation*}
    Moreover, for every instance, with probability at least $1-\delta$, the regret can be bounded as:
    \begin{equation*}
        R_T \le \widetilde{\mathcal{O}}\left(\frac{1}{(\rho-\delta)} K\sqrt{KT \log\left(\frac{1}{\delta}\right)} \right).
    \end{equation*}
\end{restatable}
The instance-dependent guarantee of \algname{} matches, up to constant factors, that of Algorithm~1 in~\cite{esfandiari2022replicable}. To the best of our knowledge, the only work providing worst-case guarantees for replicable MABs is~\citep{komiyama2024replicability}. \algname{} provides worst-case guarantees matching, up to universal constants, those of the main algorithm presented in their Theorem~11. We remark that, among these results, ours are the only guarantees that hold with high probability, rather than in expectation. Finally, we underline that the $\nicefrac{1}{\rho}$ dependence in the regret bound is necessary as shown in~\citep{ahmadi2024replicable}.

\section{Replicability in MABs With Soft Constraints}\label{sec:main}

In this section, we address the problem of designing a $\rho$-replicable learning algorithm for stochastic MABs with \emph{soft constraints}. Our approach, called \algnamesoft{}, builds on top of the previously introduced \algname{}, integrating additional tools to control the violation incurred by the algorithm by restricting the decision space.

\subsection{Algorithm}
\NewEnviron{highlightblue}{
\noindent
  \tikz[]{
    \node[
      fill=lightblue!30,
      rounded corners,
      inner sep=0pt,
      text width=\linewidth,
      align=left
    ] (X) {\BODY};
  }
}
\NewEnviron{highlightgreen}{
\noindent
  \tikz[]{
    \node[
      fill=green!30,
      rounded corners,
      inner sep=0pt,
      text width=\linewidth,
      align=left
    ] (X) {\BODY};
  }
}
\begin{algorithm}[!htp]\caption{\algnamesoft}\label{alg:main}
\begin{algorithmic}[1]
    \Require $T \in \mathbb{N}_{>0}$, $K > 1$, $0 < 2\delta < \rho < 1, \xi$
    \State $\delta' \gets \frac{\delta}{2(m+1)K^2 \lceil \log_2(T) \rceil}$,
    $\rho' \gets \frac{\rho}{2(m+1)K^2 \lceil \log_2(T) \rceil}$ \label{line:unconstrained_init_1}
    \State $h \gets 0$
    \State $N_1(a) \gets 0, \; \widetilde{N}_0(a) \gets 0, \; \widehat{\vr}_0(a) \gets \frac 1 2 \quad \forall a \in [K]$ 
    \State \hspace{-0.1cm}\tcbox[colback=lightblue!30]{$\vx^\circ_0 \gets \argmax_{\vx \in \Delta_K} (\widehat{\vr} +\zetavec_0)^\top \vx$}
    \For{$t=1,\dots,T$}
        \If{\tcbox{$\exists a\in [K] : N_t(a) \ge \max\{2 \widetilde{N}_h(a),1\}$}}
        \label{line:soft_constrained_confront}
            \State $h \gets h+1$ \label{line:unconstrained_epoch_end}
            \State $\widetilde{N}_h(a) \gets N_t(a) \quad \forall a \in [K]$ 
            \State $\mathcal{S}^{\text{r}}_{a}(t) \gets \{r_{\tau}(a)\mid \tau < t, a_{\tau}=a\}\;\;\;\;\;\; \tcbhighmath{\forall a \in [K]}$
            \State $\widehat{\vr}_h(a) \gets \ReprMean( \mathcal{S}^{\text{r}}_{a}(t), \delta', \rho', \xi )\;\;\;\;\;\; \tcbhighmath{\forall a \in [K]}
            $\label{line:soft_constrained_ucb}
            \Statex \begin{highlightblue}
            \State $\mathcal{S}^{\text{g}}_{a,i}(t) \gets \{g_{\tau,i}(a)\mid \tau < t, a_{\tau}=a\} \quad \forall a \in [K]$
            \State $\widehat{\vg}_{h,i}(a) \gets  \ReprMean( \mathcal{S}^{\text{g}}_{a,i}(t), \delta', \rho' , \xi) ~~ \forall a \in [K]$\label{line:soft_constrained_ucb_1}
            \State $\underline{\mathcal{X}}_h \gets \{\vx \in \Delta_K \mid (\widehat{\vg}_{h,i} - \zetavec_h)^\top \vx \le \alpha_i \; \forall i \in [m]\}$\label{line:soft_constrained_ucb_2}
            \State $\vx^\circ_h \gets \argmax_{\vx \in \underline{\mathcal{X}}_h} (\widehat{\vr}_h +\zetavec_h)^\top \vx$\label{line:soft_constrained_ucb_3}
            \end{highlightblue}
        \EndIf
        \State \hspace{-0.1cm}\tcbox{Select $a_t \sim \vx^\circ_h$}
        \label{line:unconstrained_choice_1}
        \State Receive $r_t(a_t), g_{t,i}(a_t) \quad \forall i \in [m]$
        \State $N_{t+1}(a) \gets N_t(a) + \mathbbm{1}\{a=a_t\} \quad \forall a \in [K]$ 
        \label{line:unconstrained_choice_2}
    \EndFor
    \end{algorithmic}
\end{algorithm}

Algorithm~\ref{alg:main} provides the pseudocode of \algnamesoft{}. This algorithm builds on \algname{}, sharing the epoch-based UCB approach. The presence of constraints precludes the algorithm from choosing the action that maximizes the UCB index, and instead restricts the decision space from the full simplex to a convex polytope. In Algorithm~\ref{alg:main}, we highlight in \textcolor{cyan}{light-blue} the block containing the set of instructions necessary to control the violation. At the beginning of every epoch $h$, for every constraint $i \in [m]$ and action $a \in [K]$, \algnamesoft{} estimates the average cost via a $\rho'$-replicable estimator $\widehat{\vg}_{h,i}(a)$ (Line~\ref{line:soft_constrained_ucb_1}). Then, these estimates are used together with their confidence interval widths $\boldsymbol{\zeta}_h$ to compute an optimistic \textit{safe} decision space:
\begin{equation}
\label{eq:optimistic_safe_set_def}
\underline{\mathcal{X}}_h = \{\vx \in \Delta_K \mid (\widehat{\vg}_{h,i} - \zetavec_h)^\top \vx \le \alpha_i \; \forall i \in [m]\},
\end{equation}
where the estimation uncertainty controls the expected violation. Let us remark that the optimistic safe decision space computed by \algnamesoft{} is guaranteed to contain the true safe decision space with high probability, as stated in the following lemma.
\begin{restatable}{lemma}{SafeSetLemma}
    \label{lem:safe_set}
    Let $\mathcal{X}^\star \coloneqq \{\vx \in \Delta_K \mid \vg_{i}^\top \vx \le \alpha_i \; \forall i \in [m]\}$ be the set of safe strategies.
    With probability at least $1-\delta$, for every epoch $h$ it holds that $\mathcal{X}^\star \subseteq\XOptSafe_h  $, where $\XOptSafe_h$ is defined in Equation~\eqref{eq:optimistic_safe_set_def}.
\end{restatable}
Notice that, as the number of epochs increases and more data about costs are collected, the set $\XOptSafe_h$ \textit{shrinks} toward the set $\mathcal{X}^\star$ of truly safe strategies. In this way, we avoid excluding the optimal safe strategy $\vx^\star$, while the incurred per-round violation becomes increasingly close to that of safe strategies. Finally, the proposed strategy $\vx_h^\circ$ is the one that maximizes the optimistic reward within $\underline{\mathcal{X}}_h$ (Line~\ref{line:soft_constrained_ucb_3}).

The aforementioned reasoning is one the main motivation behind the use of UCB-style algorithms for constrained MABs. Unlike unconstrained settings, where the learner deals with a single objective and can discard arms based on a unique reward signal, the constrained setting presents a more complex trade-off. Indeed, standard arm elimination techniques fail to cope with the tension between multiple objectives. For instance, since constraints estimation can potentially classify high-reward regions of the simplex as infeasible, reward-based elimination could prematurely discard optimal safe solutions that appear suboptimal in terms of rewards. Additionally, even if such solutions were reintroduced at a later stage of the learning process, the corresponding region of the simplex would still not have been explored sufficiently. This invalidates the standard analysis of arm elimination algorithms, which typically assumes that all active actions are sampled roughly uniformly.

It is worth noticing some additional, but crucial, differences from \algname{}, which are highlighted in \textcolor{green!50!black}{green} in Algorithm~\ref{alg:main}. Differently from the unconstrained setting, here the algorithm chooses a strategy, and the selected action is sampled accordingly (Line~\ref{line:unconstrained_choice_1}). 
This aspect further motivates the use of UCB-style approaches, which---unlike action elimination methods---can be easily extended to operate over an arbitrary space of strategies. Moreover, the strategy selection implies that, during an epoch, every action with positive probability can be selected. As a consequence, the epoch-termination condition (Line~\ref{line:soft_constrained_confront}) accounts for every action, and the average reward estimator has to be computed for every action in every epoch (Line~\ref{line:soft_constrained_ucb}), or at least for every action that has been played at least once. 

As for \algname{}, the number and the duration of the phases are not decided beforehand, but it is possible to show that there are at most $H \le \mathcal{O}(\log T)$ phases. In every epoch, \algnamesoft{} computes $(m+1)K$ estimators, and the algorithm is $\rho$-replicable as long as each estimator is $\rho'$-replicable, with $\rho' \le \rho / \big((m+1)K \log T\big)$. Analogously to \algname{}, we summarize the bound on the number of epochs and the $\rho$-replicability of \algnamesoft{} in a lemma.
\begin{restatable}{lemma}{ReplicableSoftLemma}
    \label{lem:replicability_and_epochs}
    \algnamesoft executes $H \le K\lceil\log_2(T)\rceil$ epochs and is $\rho$-replicable.
\end{restatable}

\NewEnviron{highlightpurple}{
\noindent
  \tikz[]{
    \node[
      fill=purple!10,
      rounded corners,
      inner sep=0pt,
      text width=\linewidth,
      align=left
    ] (X) {\BODY};
  }
  }
\tcbset{on line, 
        boxsep=1pt, left=0pt,right=0pt,top=0pt,bottom=0pt,
        colframe=white,colback=purple!10,  
        highlight math style={enhanced}
        }
\begin{algorithm}[H]
\caption{\algnamehard}\label{alg:main_hard_constr}
\begin{algorithmic}[1]
    \Require $T \in \mathbb{N}_{>0}$, $K > 1$, \tcbox{strictly safe strategy $\vx^\diamond$}, $\delta>0$, $\rho > 2\delta$, $\xi$\label{line:req_hard}
    \State $\delta' \gets \frac{\delta}{2(m+1)K \lceil \log_2(T) \rceil}$,
    $\rho' \gets \frac{\rho}{2(m+1)K \lceil \log_2(T) \rceil}$ 
    \State $h \gets 0$
    \State $N_1(a) \gets 0, \; \widetilde{N}_0(a) \gets 0, \; \widehat{\vr}_0(a) \gets \frac 1 2 \quad \forall a \in [K]$ 
    \State $\sigma_0 \gets \max_{i \in [m]} \frac{1-\alpha_i}{1 -\alpha_i +\lambda_i}$ \label{line:init_hard_a}
    \State $\vx^\circ_0 \gets \sigma_h \vx^\diamond + (1-\sigma_h)\widetilde{\vx}_h$.$\phantom{\big(}$ \label{line:init_hard_b}
    \For{$t=1,\dots,T$}
        \If{$\exists a\in [K] : N_t(a) \ge \max\{2 \widetilde{N}_h(a),1\}$} 
            \State $h \gets h+1$ 
            \State $\widetilde{N}_h(a) \gets N_t(a) \quad \forall a \in [K]$ 
            \State $\mathcal{S}^{\text{r}}_{a}(t) \gets \{r_{\tau}(a)\mid \tau < t, a_{\tau}=a\}\;\;\;\;\;\; \forall a \in [K]$
            \State $\widehat{\vr}_h(a) \gets \ReprMean( \mathcal{S}^{\text{r}}_{a}(t), \delta', \rho', \xi )\;\;\;\;\;\; \forall a \in [K]$
            \Statex \begin{highlightblue}
            \State $\mathcal{S}^{\text{g}}_{a,i}(t) \gets \{g_{\tau,i}(a)\mid \tau < t, a_{\tau}=a\} \quad \forall a \in [K]$
            \State $\widehat{\vg}_{h,i}(a) \gets  \ReprMean( \mathcal{S}^{\text{g}}_{a,i}(t), \delta', \rho',\xi) ~~\forall a \in [K]$\label{line:g_hard_safe}
            \State $\underline{\mathcal{X}}_h \gets \{\vx \in \Delta_K \mid (\widehat{\vg}_{h,i} - \zetavec_h)^\top \vx \le \alpha_i \; \forall i \in [m]\}$
            \State $\widetilde{\vx}_h \gets \argmax_{\vx \in \underline{\mathcal{X}}_h} (\widehat{\vr}_h +\zetavec_h)^\top \vx$ \label{line:soft_hard} 
            \end{highlightblue}
            \Statex \begin{highlightpurple}
            \State $\mathcal{V}_h \gets \{i \in [m] \mid (\widehat{\vg}_{h,i} +\zetavec_h)^\top \widetilde{\vx}_h > \alpha_i\}$ \label{line:hard_constraints_1}
            \If{$\mathcal{V}_h \neq \varnothing$}\label{line:hard_constraints_2a}
            \State $\sigma_h \gets \max_{i \in \mathcal{V}_h} \frac{\min\{ (\widehat{\vg}_{h,i} +\zetavec_h)^\top \widetilde{\vx}_h,1\}-\alpha_i}{\min\{(\widehat{\vg}_{h,i} +\zetavec_h)^\top \widetilde{\vx}_h,1\} -\alpha_i +\lambda_i}
            \label{line:hard_constraints_2b}
            $
            \Else \label{line:hard_constraints_2c}
            \State $\sigma_h \gets 0$ \label{line:hard_constraints_2d}
            \EndIf
            \State $\vx^\circ_h \gets \sigma_h \vx^\diamond + (1-\sigma_h)\widetilde{\vx}_h$ \label{line:hard_constraints_3}
        \end{highlightpurple}
        \EndIf
        \State Select $a_t \sim \vx^\circ_h$
        \State Receive $r_t(a_t), g_{t,i}(a_t) \quad \forall i \in [m]$
        \State $N_{t+1}(a) \gets N_t(a) + \mathbbm{1}\{a=a_t\} \quad \forall a \in [K]$ 
    \EndFor
    \end{algorithmic}
\end{algorithm}

We are now ready to state the main results of this section, which are an upper bound on the regret and an upper bound on the cumulative violation of \algnamesoft.
\begin{restatable}{theorem}{MainTheorem}
    \label{th:main_theorem_soft}
    Let $\delta \in (0,1)$ and $\rho \in (2\delta, 1)$. Then, with probability at least $1-\delta$, \algnamesoft suffers an expected cumulative regret bounded as:
    \begin{equation*}
        R_T \le \widetilde{\mathcal{O}}\left(\frac{1}{\rho-2\delta} \log\left(\frac{1}{\delta}\right) mK^2\sqrt{KT}  \right),
    \end{equation*}
    and a cumulative positive constraint violation bounded as:
    \begin{equation*}
        V_T \le \widetilde{\mathcal{O}}\left(\frac{1}{\rho-2\delta} \log\left(\frac{1}{\delta}\right) mK^2\sqrt{KT}  \right).
    \end{equation*}
\end{restatable}
The regret of \algnamesoft{} can be upper bounded in the same way as the regret of \algname{}, matching the rate of the state-of-the-art result~\cite{esfandiari2022replicable}, up to a factor of $m$ due to the presence of constraints. The cumulative violation is upper bounded of the same order as the cumulative regret and are thus sublinear.
Compared to the state-of-the-art results for constrained MABs (and constrained MDPs) with \emph{soft} constraints provided in~\citep{Exploration_Exploitation}, our algorithm attains the optimal rate in $T$ for both regret and violation. 
Finally, we underline that, since constrained MABs are a generalization of unconstrained ones, the dependence in $\nicefrac{1}{\rho}$ is unavoidable for constrained settings, too.

\section{Replicability in MABs with Hard Constraints}
\label{sec:hard_constraints}

In this section, we design a $\rho$-replicable learning algorithm for stochastic MABs with \emph{hard constraints}. We propose \algnamehard, an algorithm that puts an additional layer over \algnamesoft, by interpolating the proposed strategy with a provided strictly safe strategy.

\subsection{Algorithm}

Algorithm~\ref{alg:main_hard_constr} reports the pseudocode of \algnamehard{}, a $\rho$-replicable algorithm for MABs with hard constraints. \algnamehard{} builds directly on top of \algnamesoft{}, by adding a block of instructions (highlighted in \textcolor{purple!50}{pink}) that uses Slater's condition and the strictly safe strategy $\vx^\diamond$, which is required as an input, in order to enforce safety over the proposed strategy $\vx_h^\circ$ for every epoch $h$.

Let us remark that optimizing inside the safe set employed by \algnamesoft{} would result in a strategy $\widetilde{\vx}_h$ (Line~\ref{line:soft_hard}) that may be unsafe during each of the rounds of epoch $h$. For instance, in the first rounds, the optimistic safe set employed in \algnamesoft{} is equivalent to the entire simplex. In such a case, the first strategy played by \algnamesoft{} can be trivially unsafe. Furthermore, even in later rounds, \algnamesoft{} tends to play at the frontier of the decision space, due to the linearity of the objective function, which is included thank to the optimism. Accordingly, since the frontier is the unsafest part of the estimated decision space (this is particularly true when the optimal solution is mixed), \algnamesoft{} can play unsafe strategies even in the last rounds of the learning dynamics.
To tackle this problem, we define the set of constraints:
\begin{equation}
    \label{eq:pessimistic_violations_def}
    \mathcal{V}_h \coloneqq \{i \in [m] \mid (\widehat{\vg}_{h,i} +\zetavec_h)^\top \widetilde{\vx}_h > \alpha_i\},
\end{equation}
which the strategy $\widetilde{\vx}_h$ \emph{may} violate during epoch $h$ (Line~\ref{line:hard_constraints_1}). This is done to adopt a \textit{pessimistic} perspective by checking which constraints might be violated with probability at least $1-\delta$ during epoch $h$, and by \textit{how much} they are violated in the worst-case.
If there exists one or more constraints that are at risk of violation, \algnamehard computes a combination factor: 
\begin{equation}
    \label{eq:sigma_h_def}
    \sigma_h = \max_{i \in \mathcal{V}_h} \frac{\min\{ (\widehat{\vg}_{h,i} +\zetavec_h)^\top \widetilde{\vx}_h,1\}-\alpha_i}{\min\{(\widehat{\vg}_{h,i} +\zetavec_h)^\top \widetilde{\vx}_h,1\}-\alpha_i +\lambda_i},
\end{equation}
by weighting the worst-case violation of a constraint over itself plus the margin of the strictly safe strategy $\vx^\diamond$ for that specific constraint, and taking the maximum of this weight over the constraints in $\mathcal{V}_h$ (Lines~\ref{line:hard_constraints_2a} and~\ref{line:hard_constraints_2b}). If no constraint is considered at risk of violation (\emph{i.e.}, $\mathcal{V}_h = \varnothing$), then \algnamehard sets $\sigma_h = 0$ (Lines~\ref{line:hard_constraints_2c} and~\ref{line:hard_constraints_2d}).
The actual strategy $\vx^\star_h$ employed by \algnamehard during epoch $h$ is computed as a linear combination of $\vx^\diamond$ and $\widetilde{\vx}_h$, weighted by $\sigma_h$, formally,
    $\vx^\star_h = \sigma_h \vx^\diamond + (1-\sigma_h)\widetilde{\vx}_h$ is selected.
Notice that this procedure is also applied to the very first strategy played by \algnamehard (Lines~\ref{line:init_hard_a} and~\ref{line:init_hard_b}), ensuring safety since the beginning of the learning process.
Thanks to this procedure, \algnamehard is a safe algorithm for MABs with hard constraints.
\begin{restatable}{lemma}{HardSafety}
    \label{lem:hard_safety}
    With a probability of at least $1-\delta$, \algnamehard plays a safe strategy $\vx_t$ at every round $t \in [T]$.
\end{restatable}
As for \algname{} and \algnamesoft{}, in the following lemma we state a bound on the number of epochs for \algnamehard{}, together with its $\rho$-replicability.
\begin{restatable}{lemma}{ReplicableHardLemma}
    \label{lem:replicability_and_epochs_hard}
     \algnamehard executes $H \le K\lceil\log_2(T)\rceil$ epochs and is $\rho$-replicable.
\end{restatable}
Finally, we state the regret upper bound of \algnamehard.
\begin{restatable}{theorem}{RegretHardHighProb}
    Let $\delta \in (0,1)$ and $\rho \in (2\delta, 1)$. Then, with probability at least $1-\delta$, \algnamehard suffers an expected cumulative regret bounded as:
    \begin{equation*}
        R_T \le \widetilde{\mathcal{O}} \left(\frac{1}{\lambdaMin(\rho-2\delta)} \log\left(\frac{1}{\delta}\right) mK^2\sqrt{KT} +\gamma \right)
    \end{equation*}
    where $\gamma$ does not depend on $T$.
\end{restatable}
Comparing the performance of Algorithm~\ref{alg:main_hard_constr} with the state-of-the-art, our algorithm matches the regret bound of~\citep{pacchiano} for the MAB case, in terms of $T$. To conclude, we underline that the $\nicefrac{1}{\lambda_{\min}}$ dependence in the regret bound is provably necessary~\citep{stradi2024learning}. 


\bibliographystyle{plainnat}
\bibliography{example_paper}

\newpage
\appendix
\tableofcontents

\section{Additional Related Works}
\label{App:related}

In this section, we provide additional discussion on the constrained online learning literature.

Online learning with \emph{unknown} constraints has been extensively studied (see, \emph{e.g.},~\citep{Mannor, liakopoulos2019cautious, Unifying_Framework, stradi2025noregret}). Two main settings are typically considered. In \emph{soft-constraint} regimes (see, \emph{e.g.},~\citep{chen2022strategies}), the goal is to ensure that the algorithm’s cumulative constraint violations grow sub-linearly. In \emph{hard-constraint} regimes, instead, algorithms are required to satisfy the constraints at every round, usually under the assumption that a strictly feasible decision is known in advance (see, \emph{e.g.},~\citep{pacchiano}). Both soft and hard constraints have also been extended to settings that are more challenging than MABs, such as linear bandits (see, \emph{e.g.},~\citep{gangrade2024safe}) and games (see, \emph{e.g.},~\citep{safetree}). Finally, stronger notions of regret bound (such as data-dependent ones) have been provided for both \emph{soft} and \emph{hard} constraints in~\citep{genalti2025datadependent}.
To the best of our knowledge, none of these contributions in the constrained online learning literature focus on replicability.

A strongly related body of work is the one on online learning in constrained Markov decision processes (see, \emph{e.g.},~\citep{Online_Learning_in_Weakly_Coupled, Constrained_Upper_Confidence, bai2020provably, Exploration_Exploitation, Upper_Confidence_Primal_Dual, stradi2024, stradi2024best, muller2024truly, stradioptimal,stradi2025taming}). The most related work is~\citep{stradi2024learning}, where both \emph{soft} and \emph{hard} constraints are studied in CMDPs with adversarial losses and stochastic constraints. Nonetheless, none of these works focus on developing replicable algorithms.

\section{Additional Notation}
We introduce some additional notation that we will use throughout the following sections.
We let $h(t)$ be the epoch associated with round $t \in [T]$.
Specifically, when we analyze Algorithm~\ref{alg:main_unconstrained} we let $h(t) \in \mathbb{N}$ be the value such that the algorithm plays action $a_t = a^\circ_h$ at round $t$.
Similarly, analyzing Algorithm~\ref{alg:main} and Algorithm~\ref{alg:main_hard_constr} we let $h(t)$ be such that the algorithm plays according to strategy $\vx_t = \vx^\circ_h$ at round $t$.
With this notation, we can formally define the first round of epoch $h$ as $t_h \coloneqq \min \{t \in[T] \mid h(t)=h\}$ for every epoch $h \in \{0,\,\dots,H\}$.
Let us also recall that for every action $a \in [K]$, round $t \in [T]$ and epoch $h \in \{0,\,\dots,H\}$ the quantities $N_t(a)$ and $\widetilde{N}_t(a)$ are defined as:
\begin{align*}
    N_t(a) &\coloneqq \sum_{j=1}^{t-1} \mathbb{I}\{a_j = a\}, \\
    \widetilde{N}_t(a) &\coloneqq N_{t_h}(a) = \sum_{j=1}^{t_h-1} \mathbb{I}\{a_j = a\},
\end{align*}
where $N_t(a)$ is the number of times action $a$ has been selected up to round $t$ \emph{excluded}, 
Notice that all the algorithms presented in the paper correctly keep track of the quantities defined above.

For the ease of exposition, we also define the event under which the estimators computed by Algorithm~\ref{alg:main} and Algorithm~\ref{alg:main_hard_constr} belong to the appropriate confidence interval.
Formally, we define the \emph{clean event} $\mathcal{E}$ as follows:
\begin{definition}[Clean Event]
    \label{def:clean_event}
    We let $\mathcal{E}$ be the event under which, for every epoch $h \ge 1$ and arm $a \in [K]$, it holds that:
    \begin{equation}
        \label{eq:clean_r}
        \widehat{\vr}_h(a) -\zetavec_h(a) \le \vr(a) \le \widehat{\vr}_h(a) +\zetavec_h(a)
    \end{equation}
    and 
    \begin{equation}
        \label{eq:clean_g}
        \widehat{\vg}_{h,i}(a) -\zetavec_h(a) \le \vg(a) \le \widehat{\vg}_{h,i}(a) +\zetavec_h(a)
    \end{equation}
    for every constraint $i \in [m]$.
\end{definition}

For the analysis of Algorithm~\ref{alg:main_hard_constr}, we also let $\XPessSafe_h$ be the pessimistic set of safe strategies defined as:
\begin{equation}
    \label{eq:pessimistic_safe_set_def}
    \XPessSafe_h \coloneqq \{\vx \in \Delta_K \mid (\widehat{\vg}_{h,i} +\zetavec_h)^\top \vx \le \alpha_i \quad \forall i \in [m] \},
\end{equation}
which contains only safe strategies under the event $\mathcal{E}$.
\section{Proofs omitted from Section~\ref{sec:unconstrained}}
\label{apx:unconstrained}

\ReplicableUnconstrainedLemma*
\begin{proof}
    When $T=1$, there is just a single epoch $h=0 \le K\lceil \log_2(T) \rceil$ and Algorithm~\ref{alg:main_unconstrained} deterministically chooses the strategy to play during the single round. 
    We will therefore consider the non-trivial case where $T \ge 2$.
    Furthermore, we will denote with $H$ be the last epoch.

    For every $a \in[K]$, let $\mathcal{H}_a$ be the set of epochs $h \in [2,H]$ such that $a^\circ_{h-1} = a$.
    Denoting with $h_a(i)$ the  $i$-th element of $\mathcal{H}_a$ in ascending order, we show that $\widetilde{N}_{h_a(i)}(a) = 2^{i-1}$.
    When $i=1$, we have $\widetilde{N}_{h_a(1)}(a) = 1$ as an epoch terminates when an action is played for the first time.
    Instead, when $i>1$, we observe that:
    \begin{equation*}
        \widetilde{N}_{h_a(i)}(a) = 2\widetilde{N}_{h_a(i)-1}(a) = \widetilde{N}_{h_a(i-1)}(a).
    \end{equation*}
    As a result, it holds that $\widetilde{N}_{h_a(i)}(a) = 2^{i-1}$.
    Since $T \ge \widetilde{N}_{h_a(i)}(a)$, we also get $i \le \log_2(T) + 1$, which in turn implies that $\mathcal{H}_a \le \lfloor\log_2(T) + 1\rfloor \le \lceil \log_2(T)\rceil$.
    Finally, we can upper bound the number of epochs of Algorithm~\ref{alg:main_unconstrained} as:
    \begin{equation*}
        H \le \sum_{a \in [K]} |\mathcal{H}_a| \le \sum_{a \in [K]} \lceil \log_2(T)\rceil \le K\lceil \log_2(T) \rceil.
    \end{equation*}
    

    Now we show that Algorithm~\ref{alg:main_unconstrained} is $\rho$-replicable.
    Consider two different trials of Algorithm~\ref{alg:main_unconstrained}.
    Let $H_i$ be the last epoch of trial $i$, and $a^i_0, a^i_1, \dots a^i_{H_i}$ and $\widehat{\vr}^i_0,\widehat{\vr}^i_1,\dots,\widehat{\vr}^i_{H_i}$ be the sequence of actions and estimators computed at each epoch of trial $i \in \{1,2\}$, respectively.
    By construction, Algorithm~\ref{alg:main_unconstrained} plays the same sequence of actions in both trials if $a^1_h = a^2_h$ for every $h \in [H_1]$ -- notice that in this case $H_1 = H_2$.
    
    Initially (epoch $h=0$), Algorithm~\ref{alg:main_unconstrained} plays always the same action $a^\circ_0 = a^1_0 = a^2_0$ computed without any sample from the environment. 
    Hence the algorithm plays the same actions during the first epoch in both trials, which also always lasts for a single round only.
    Furthermore, the estimators $\widehat{\vr}^1_0$ and $\widehat{\vr}^2_0$ coincide.

    Now suppose, by induction, that with  probability at least $1-(h-1)\rho'$ a phase $h \ge 1$ begins at rounds $\bar t \in [T]$ in both trials, and that $a^1_{h'} = a^2_{h'} \eqqcolon a_{h'}$ and $\widehat{\vr}^i_{h'} = \widehat{\vr}^i_{h} \eqqcolon \vr_{h'}$ for every previous epoch $0 \le h' < h$.
    Observe that this inductive hypothesis holds for epoch $h = 1$.
    For every action $a \neq a_{h-1}$ Algorithm~\ref{alg:main_unconstrained} computes the same estimator $\widehat{\vr}^1_h(a) = \widehat{\vr}^2_h(a) = \widehat{\vr}_{h-1}(a)$.
    For action $a = a_{h-1}$, it employs a $\rho'$-replicable estimator (Lemma~\ref{lem:replicable_estiamtor}).
    Therefore, the estimators $\widehat{\vr}^1$ and $\widehat{\vr}^2$ coincide with probability at least $1-h\rho'$.
    When this happens, the Algorithm selects the same action $a^1_h = a^2_h$ to be played at epoch $h$.
    As a result, the epoch lasts the same number of rounds in both trials with probability at least $1-h\rho'$ and the inductive hypothesis holds at epoch $h+1$.

    Since there are $K\lceil\log_2(T)\rceil$ epochs, the algorithm  plays the same sequence of actions in both trials with probability at least $1-K\lceil\log_2(T)\rceil \rho'$.
    Taking $\rho' \le \frac{\rho}{K\lceil\log_2(T)\rceil}$, the algorithm is $\rho$-replicable.
\end{proof}
\begin{restatable}{lemma}{UnconsrainedNumberPullsLemma}
    \label{lem:number_pulls_unconstrained}
    During the execution of Algorithm~\ref{alg:main_unconstrained}, each suboptimal action $a \in [K]$ such that $\Delta(a) > 0$ is played at most:
    \begin{equation*}
        N_{T+1}(a) \le \frac{15 \ln\left( \frac{1}{\delta'} \right)}{\Delta(a)^2 (\rho'-2\delta')^2}
    \end{equation*}
    times with probability at least $1-\delta$.
\end{restatable}
\begin{proof}
    We first prove the upper bound on the number of pulls conditioned under the event that $|\vr(a) -\widehat{\vr}_h(a)| \le \zetavec_h(a)$ for every $a \in [K]$.
    We will later show that this event has probability at least $1-\delta$.
    
    Consider a suboptimal action $a \in [K] : \Delta(a)>0$ pulled at least once during the execution of Algorithm~\ref{alg:main_unconstrained} (the upper bound trivially holds otherwise).
    Let $h$ be the last epoch where action $a$ is selected, \emph{i.e.} $a^\circ_h = a$ and $a^\circ_{h'} \neq a$ for every subsequent epoch $h' > h$.
    For such an action to be selected at an epoch $h$, it must be the case that:
    \begin{equation*}
        \widehat{\vr}_h(a) +\zetavec_h(a) \ge \widehat{\vr}_h(a^\star) +\zetavec_h(a^\star).
    \end{equation*}
    We further observe that, under the event that each estimator lies within the appropriate confidence interval, we have:
    \begin{align*}
        \vr(a) +2\zetavec_h(a) \ge \widehat{\vr}_h(a) +\zetavec_h(a) 
        \ge \widehat{\vr}_h(a^\star) +\zetavec_h(a^\star) 
        \ge \vr(a^\star).
    \end{align*}
    Therefore:
    \begin{equation*}
        \zetavec_h(a) \ge \frac{\Delta(a)}{2}.
    \end{equation*}
    By the definition of $\zetavec_h(a)$ (Equation~\eqref{eq:zetavec_def}) and simple computations we get:
    \begin{equation*}
        \widetilde{N}_h(a) \le \max\{\widetilde{N}_h(a),1\} \le \frac{8 \ln\left( \frac{2}{\delta'} \right)}{\Delta(a)^2(\rho'-2\delta')^2}.
    \end{equation*}
    We now observe that since $h$ is the last epoch where action $a$ is selected, we have $N_{T+1}(a) \le \max\{2\widetilde{N}_h(a),1\}$.
    Indeed, the epoch terminates when action $a$ has been played $\max\{2\widetilde{N}_h(a),1\}$ times, and such an action is never played afterward.
    As a result:
    \begin{equation*}
        N_{T+1}(a) \le \max\{2\widetilde{N}_h(a),1\} \le \frac{16 \ln\left( \frac{2}{\delta'} \right)}{\Delta(a)^2(\rho'-2\delta')^2},
    \end{equation*}
    where we observe that:
    \begin{equation*}
        \frac{16 \ln\left( \frac{2}{\delta'} \right)}{\Delta(a)^2(\rho'-2\delta')^2} \ge \frac{16 \ln\left( \frac{2}{\delta'} \right)}{(1-2\delta')^2} \ge 1
    \end{equation*}
    for every $\delta' \in (0,1)$.

    In order to conclude the proof, we show that that $|\vr(a) -\widehat{\vr}_h(a)| \le \zetavec_h(a)$ for every $a \in [K]$ with probability at least $1-\delta$.
    Suppose, by induction, that at some epoch $1 \le h < H$ it holds $|\vr(a) -\widehat{\vr}_h(a)| \le \zetavec_h(a)$ for every $a \in [K]$  with probability at least $1-(h-1)\delta'$ --- $H$ denotes the last epoch.
    Algorithm~\ref{alg:main_unconstrained} plays always the same action $a^\circ_h$ for the entire epoch $h$.
    Subsequently, it computes a new estimator $\widehat{\vr}_{h+1}(a^\circ_h)$ which according to Lemma~\ref{lem:replicable_estiamtor} lies at distance at most $\zetavec_{h+1}(a^\circ_h)$ from $\vr(a)$ with probability at least $1-\delta'$.
    For every other action $a \neq a^\circ_h$, it computes the estimator $\widehat{\vr}_{h+1}(a)=\widehat{\vr}_h(a)$, which lies at distance $\zetavec_h(a)$ from $\vr(a)$ with probability at least $1-(h-1)\delta'$ by the inductive hypothesis.
    We observe that $\widetilde{N}_{h+1}(a)=\widetilde{N}_h(a)$, as the action has not been pulled during the epoch $h$.
    Therefore, $\zetavec_h(a) = \zetavec_{h+1}(a)$ by Equation~\eqref{eq:zetavec_def}, and the estimator lies within the correct interval with probability at least $1-(h-1)\delta'$.
    By an union bound, every estimator lies within the correct interval with probability at least $1-h\delta'$, and the inductive hypothesis holds at epoch $h+1$.
    Furthermore, the inductive hypothesis holds at the epoch $h=1$, as the algorithm plays deterministically at epoch $h=0$ and $\zetavec_1(a) \ge 1$ for every action $a \in [K]$.
    As a result, since there are at most $K\lceil\log_2(T)\rceil$ epochs (Lemma~\ref{lem:replicability_and_epochs_unconstrained}), with probability at least $1-K\lceil\log_2(T)\rceil\delta'$ every estimator computed by Algorithm~\ref{alg:main_unconstrained} lies within the appropriate confidence interval.
    The proof is concluded by observing that $\delta'\le\delta/(K\lceil\log_2(T)\rceil)$.
\end{proof}

\RegretUnconstrained*
\begin{proof}
    Thanks to Lemma~\ref{lem:number_pulls_unconstrained}, the number of times suboptimal action $a \in [K]$ is played is bounded by:
    \begin{equation*}
        N_{T+1}(a) \le \frac{15 \ln\left( \frac{1}{\delta'} \right)}{\Delta(a)^2 (\rho'-2\delta')^2}
    \end{equation*}
    with probability at least $1-\delta$.
    Let us continue the analysis conditioned on such an event.
    
    Let $\mathcal{S} \subseteq [K]$ be a set of actions such that $\Delta(a) > 0$ for every $a \in \mathcal{S}$.
    We have that:
    \begin{align}
        \sum_{a \in \mathcal{S}} \Delta(a)N_{T+1}(a)
        &\le \mathcal{O} \left( \frac{1}{(\rho'-2\delta')^2}\log\left( \frac{1}{\delta'} \right) \sum_{a \in \mathcal{S}} \frac{1}{\Delta(a)} \right)  \nonumber\\
        &=\mathcal{O} \left( \frac{K^2 \log^2(T)}{(\rho-\delta)^2} \log\left( \frac{K\log(T)}{\delta} \right) \sum_{a \in \mathcal{S}} \frac{1}{\Delta(a)}\right). \label{eq:proof_unconstrained_regret_set}
    \end{align}
    In particular, by taking $\mathcal{S} = \{a \in [K] \mid \Delta(a) > 0\}$, we have:
    \begin{equation*}
        R_T = \sum_{a : \Delta(a)>0} \Delta(a)N_{T+1}(a) \le \mathcal{O} \left( \frac{K^2 \log^2(T)}{(\rho-\delta)^2} \log\left( \frac{K\log(T)}{\delta} \right) \sum_{a : \Delta(a)>0} \frac{1}{\Delta(a)}\right),
    \end{equation*}
    proving the instance-dependent regret bound.
    Furthermore, for every $\epsilon \in (0,1)$ we have:
    \begin{align*}
        R_T &= \sum_{a : \Delta(a) \le \epsilon} \Delta(a) N_{T+1}(a) + \sum_{a : \Delta(a) > \epsilon} \Delta(a) N_{T+1} (a) \\
        &\le \mathcal{O} \left(\epsilon T +\frac{K^2 \log^2(T)}{(\rho-\delta)^2} \log\left( \frac{K\log(T)}{\delta} \right) \sum_{a : \Delta(a)>\epsilon} \frac{1}{\Delta(a)} \right) \\
        &\le \mathcal{O} \left(\epsilon T +\frac{K^3 \log^2(T)}{(\rho-\delta)^2} \log\left( \frac{K\log(T)}{\delta} \right) \frac{1}{\epsilon} \right),
    \end{align*}
    where we employed Equation~\eqref{eq:proof_unconstrained_regret_set} with $\mathcal{S}=\{a \in [K] \mid \Delta(a) > \epsilon\}$.
    Taking:
    \begin{equation*}
        \epsilon \coloneqq \frac{K^{3/2}\log_2(T)}{(\rho-\delta)\sqrt{T}}\sqrt{\ln\left( \frac{K\log(T)}{\delta} \right)},
    \end{equation*}
    we obtain the instance-independent bound:
    \begin{equation*}
        R_T \le \mathcal{O}\left(\frac{1}{(\rho-\delta)} K\sqrt{KT \log\left( \frac{K\log(T)}{\delta} \right)} \right),
    \end{equation*}
    concluding the proof.
\end{proof}

\section{Proofs Omitted from Section~\ref{sec:main}}
In this Section we provide the proofs of the propositions and theorems of Section~\ref{sec:main}.
Specifically, we first prove Lemma~\ref{lem:replicability_and_epochs} and Lemma~\ref{lem:safe_set}.
Then, for the ease of exposition, we provide two additional lemmas, showing that conditioned on the clean event (Definition~\ref{def:clean_event}) the regret and the cumulative violation of Algorithm~\ref{alg:main} are upper bounded by the same quantity.
After another technical lemma (Lemma~\ref{lem:concentration}, which will be employs also in Appendix~\ref{App:hard}), we prove Theorem~\ref{th:main_theorem_soft}.

\ReplicableSoftLemma*
\begin{proof}
    Let $H$ be the number of epochs during the execution of Algorithm~\ref{alg:main}.
    First, we show that $H \le K\lceil\log(T)\rceil$.
    Let us define the set of rounds $\mathcal{T}_a$ as:
    \begin{equation*}
        \mathcal{T}_a \coloneqq \{t \in [T] \mid N_t(a) = \widetilde{N}_{h(t)}(a) = \max\{2\widetilde{N}_{h(t)-1}(a),1\}\}.
    \end{equation*}
    By construction, $h(t) = h(t-1)+1$ if and only if $t \in \mathcal{T}_a$ for some $a \in [K]$ (see Line~\ref{line:soft_constrained_confront}).
    Let $t_a(i)$ denote the $i$-th element of $\mathcal{T}_a$ in ascending order.
    We prove by induction that $\widetilde{N}_{h(t_a(i))}(a) \ge 2^{i-1}$.
    First, we observe that $\widetilde{N}_{h(t_a(1))}(a) = 1$, as an epoch terminates after an action is played for the first time.
    Reasoning by induction, for $i>1$ we assume that $\widetilde{N}_{h(t_a(i-1))}(a) \ge 2^{i-2}$.
    Therefore:
    \begin{equation*}
        \widetilde{N}_{t_a(i)}(a) = 2\widetilde{N}_{h(t_a(i))-1}(a) = 2\widetilde{N}_{h(t_a(i)-1)}(a) \ge 2^{i-1},
    \end{equation*}
    where the first equality is due to the definition of $\mathcal{T}_a$, the second as $\widetilde{N}_{h}(a)$ is monotone non-decreasing in $h$, and the inequality holds by induction.
    As a result, $\widetilde{N}_{h(t_a(i))}(a) \ge 2^{i-1}$.
    Since $\widetilde{N}_{h(t_a(i))}(a) \le T$, we get $i \le \log_2(T)+1$.
    Consequently, action $a$ produces at most $|\mathcal{T}_a| \le \lfloor \log_2(T)+1 \rfloor \le \lceil \log_2(T) \rceil$ changes of epoch.
    Finally, we observe that:
    \begin{equation*}
        H = \sum_{a \in [K]} |\mathcal{T}_a| \le K\lceil \log_2(T) \rceil,
    \end{equation*}
    proving the upper bound on the number fo epochs.

    Now we show that Algorithm~\ref{alg:main} is $\rho$-replicable.
    Consider two different trials of Algorithm~\ref{alg:main} and and let $\vx^1_1,\dots,\vx^1_{H^1}$ and $\vx^2_1,\dots,\vx^2_T{H^2}$ be the strategies $x^\circ_h$ selected by the algorithm in the first and second trial, respectively.
    Observe that if $\vx^i_h = \vx^1_h$ for every epoch up until $h$, then the Algorithm plays the same sequence of actions at epoch $h$ in both trials. 
    
    The algorithm plays always the same strategy during the first epoch, which lasts a single round.
    Now we reason by induction over the epochs $1 \le h \le H^1$.
    Suppose that $\vx^1_{h'} = \vx^2_{h'}$ for every $0 \le h < h'$.
    Then $\vx^1_h = \vx^2_h$ if the estimators computed by Algorithm~\ref{alg:main} at the beginning of epoch $h$ are the same in the two trials.
    According to Lemma~\ref{lem:replicable_estiamtor}, each single estimator is $\rho'$ replicable.
    Since there are $(m+1)K$ estimators, with probability at least $1-(m+1)K\rho'$ we have $\vx^1_h = \vx^2_h$.
    By a union bound over the $H \le K\lceil\log_2(T)\rceil$ epochs, we have that the algorithm selects the same sequence of strategies in both trial -- which therefore last for the same number of rounds -- with probability at least $1-(m+1)K^2\lceil\log_2(T)\rceil\rho'$.
    Hence Algorithm~\ref{alg:main} is $(m+1)K^2\lceil\log_2(T)\rceil\rho'$-replicable.
    Taking $\rho' \le \frac{\rho}{(m+1)K^2\lceil\log_2(T)\rceil}$, the algorithm is $\rho$-replicable.
\end{proof}

\SafeSetLemma*
\begin{proof}
    By Lemma~\ref{lem:replicability_and_epochs}, the number of epochs of Algorithm~\ref{alg:main} is at most $K\log_2(T)$.
    Employing Lemma~\ref{lem:replicable_estiamtor} and an union bound over the epochs and the $m$ estimators, we have that $\widehat{\vg}_{h,i}(a) -\zetavec_h(a) \le \vg_i(a) \le \widehat{\vg}_{h,i}(a) +\zetavec_h(a)$ with probability at least $1-\rho'mK\log_2(T) \ge 1-\delta$.
    Conditioned on such an event, consider any safe strategy $\vx \in \mathcal{X}^\star$.
    We observe that for every constraint $i \in [m]$:
    \begin{equation*}
        (\widehat{\vg}_{h,i} - \zetavec_h)^\top \vx \le \vg_i^\top \vx \le \alpha_i,
    \end{equation*}
    where the first inequality holds under $\mathcal{E}$, and the second due to the definition of $\mathcal{X}^\star$.
    As a result, $\vx \in \XOptSafe_h$.
\end{proof}
\begin{restatable}{lemma}{RegretLemmaCleanEvent}
    \label{lem:regret_clean_event}
    Conditioned on the event $\mathcal{E}$ (Definition~\ref{def:clean_event}), the regret of Algorithm~\ref{alg:main} satisfies:
    \begin{equation*}
        R_T \le 2\sqrt{2}\sum_{t \in [K]} \sum_{a \in [K]} \sqrt{\frac{2\log(\nicefrac{2}{\delta'})}{\max\{1,N_t(a)\}(\rho'-2\delta')^2}} \vx_t(a).
    \end{equation*}
\end{restatable}
\begin{proof}
    As a first step, we prove that, under the event $\mathcal{E}$, it holds that: $(\widehat{\vr}_{h} +\zetavec_{h})^\top \vx^\circ_h \ge \vr^\top \vx^\star$ for every epoch $h$, where $\vx^\star$ is any optimal safe strategy.
    We observe that $\vx^\star \in \XOptSafe_{h}$, as $(\widehat{\vg}_{h,i} -\zetavec_h)^\top \vx^\star \le \vg_i^\top \vx^\star \le \alpha_i$ due to Equation~\ref{eq:clean_g} and the fact that $\vx^\star$ is safe.
    By construction (see Algorithm~\ref{alg:main}), we have:
    \begin{equation*}
        \vx^\circ_h \in \argmax_{\vx \in \XOptSafe_{h}} (\widehat{\vr}_{h} +\zetavec_{h})^\top \vx.
    \end{equation*}
    Consequently:
    \begin{equation}
        \label{eq:proof_regret_ucb}
        (\widehat{\vr}_h +\zetavec_h)^\top \vx^\circ_h \ge (\widehat{\vr}_h +\zetavec_h)^\top \vx^\star \ge \vr^\top \vx^\star,
    \end{equation}
    where the first inequality holds due to the definition of $\vx^\circ_h$ and the fact that $\vx^\star \in \XOptSafe_{h}$, while the second one is true under the event $\mathcal{E}$.

    We now derive an upper bound of the regret as follows:
    \begin{subequations}
    \begin{align}
        R_T &= \sum_{t \in [T]} (\vr^\top \vx^\star - \vr^\top \vx_t) \nonumber \\
        &= \sum_{t \in [T]} (\vr^\top \vx^\star - \vr^\top \vx_{h(t)}) \label{seq:regret_soft_one} \\
        &\le \sum_{t \in [T]}((\widehat{\vr}_{h(t)} +\zetavec_{h(t)})^\top \vx_{h(t)} - \vr^\top \vx_{h(t)}) \label{seq:regret_soft_two} \\
        &\le \sum_{t \in [T]} ((\widehat{\vr}_{h(t)} +\zetavec_{h(t)})^\top \vx_{h(t)} -(\widehat{\vr}_{h(t)} -\zetavec_{h(t)})^\top \vx_{h(t)} ) \label{seq:regret_soft_three} \\
        &=2\sum_{t \in [T]} \zetavec_{h(t)}^\top \vx_{h(t)} \nonumber \\
        &=2\sum_{t \in [T]} \sum_{a \in [K]} \sqrt{\frac{2\log(\nicefrac{2}{\delta'})}{\max\{\widetilde{N}_{h(t)}(a),1\}(\rho'-2\delta')^2}} \vx_t(a) \label{seq:regret_soft_four} \\
        &\le 2\sqrt{2}\sum_{t \in [T]} \sum_{a \in [K]} \sqrt{\frac{2\log(\nicefrac{2}{\delta'})}{\max\{N_t(a),1\}(\rho'-2\delta')^2}} \vx_t(a), \label{seq:regret_soft_five} 
    \end{align}
    \end{subequations}
    where Equation~\eqref{seq:regret_soft_one} holds by construction ($\vx_t = \vx_{h(t)}$), Equation~\eqref{seq:regret_soft_two} by Equation~\eqref{eq:proof_regret_ucb}, Equation~\eqref{seq:regret_soft_three} is satisfied under $\mathcal{E}$, Equation~\eqref{seq:regret_soft_four} by the definition of $\zetavec_{h(t)}$ and the fact that $\vx_{h(t)}=\vx_t$, and Equation~\eqref{seq:regret_soft_five} as $N_t(a) \le 2\widetilde{N}_{h(t)}(a)$ for every round $t \in [T]$ and action $a \in [K]$.
\end{proof}
\begin{restatable}{lemma}{ViolationsCleanEvent}
    \label{lem:violations_clean_event}
    Under the event $\mathcal{E}$ (Definition~\ref{def:clean_event}), the violation of Algorithm~\ref{alg:main} is at most:
    \begin{equation*}
        V_T \le 2\sqrt{2}\sum_{t \in [T]} \sum_{a \in [K]} \sqrt{\frac{2\log(\nicefrac{2}{\delta'})}{\max\{1,N_t(a)\}(\rho'-2\delta')^2}} \vx_t(a).  
    \end{equation*}
\end{restatable}
\begin{proof}
    Consider a single constraint $i \in [m]$ and let us define its positive violation:
    \begin{equation*}
        V_{T,i} \coloneqq \sum_{t \in [T]} \max\{0,\vg_i^\top \vx_t - \alpha_i\},
    \end{equation*}
    so that:
    \begin{equation*}
        V_T = \max_{i \in [m]} V_{T,i}
    \end{equation*}
    We proceed to bound $V_{T,i}$ for any given $i \in [m]$.

    Under the event $\mathcal{E}$,
    the estimators $\widehat{\vg}_{h,i}$ computed by Algorithm~\ref{alg:main} satisfy:
    \begin{equation*}
        \widehat{\vg}_{h,i}(a) - \zetavec_h(a) \le \vg_i(a) \le \widehat{\vg}_{h,i}(a) + \zetavec_h(a),
    \end{equation*}
    for any $ a \in [K]$ and epoch $h \ge 0$, where $\zetavec_h(a)$ is defined in Equation~\eqref{eq:zetavec_def}.
    Therefore, for any $\vx \in \Delta(K)$:
    \begin{equation*}
        (\widehat{\vg}_{h,i} - \zetavec_h)^\top \vx \le \vg_i^\top \vx \le (\widehat{\vg}_{h,i} + \zetavec_h)^\top \vx.
    \end{equation*}
    This allows us to bound the cumulative violation as follows:
    \begin{align}
        V_{T,i} &=\sum_{t \in [T]} \max\{0,\vg_i^\top \vx_{h(t)} - \alpha_i\} \nonumber \\
        &\le \sum_{t \in [T]} \max\{0,(\widehat{\vg}_{h(t),i} + \zetavec_{h(t)})^\top \vx_{h(t)} - \alpha_i\} \nonumber \\
        &=\sum_{t \in [T]} \max\{0,(\widehat{\vg}_{h(t),i} + 2\zetavec_{h(t)} -\zetavec_{h(t)})^\top \vx_{h(t)} - \alpha_i\} \nonumber \\
        &= \sum_{t \in [T]} \max\{0,(\widehat{\vg}_{h(t),i} -\zetavec_{h(t)})^\top \vx_{h(t)} - \alpha_i + 2\zetavec_{h(t)}^\top \vx_{h(t)}\} \nonumber \\
        &\le \sum_{t \in [T]} \max\{0, 2\zetavec_{h(t)}^\top \vx_{h(t)}\} \nonumber \\
        &\le 2\sum_{t \in [T]} \zetavec_{h(t)}^\top \vx_{h(t)}, \label{eq:proof_violations_to_zeta_x}
    \end{align}
    where the second inequality holds since $\vx_{h(t)} \in \XOptSafe_{h(t)}$ (see the proof of Lemma~\ref{lem:regret_clean_event}), and therefore $\widehat{\vg}_{h(t),i} -\zetavec_{h(t)})^\top \vx_{h(t)} \le \alpha_i$, while the last one holds as every component of $\zetavec_{h(t)}$ is non-negative (see Equation~\eqref{eq:zetavec_def}).
    We can now employ the same argument similar to the one provided in the proof of Lemma~\ref{lem:regret_clean_event}.
    In particular, by employing Equation~\eqref{eq:zetavec_def}, Equation~\eqref{eq:proof_violations_to_zeta_x} and the fact that $N_t(a) \le 2 \widetilde{N}_{h(t)}(a)$ for every round $t \in [T]$ and action $a \in [K]$, we get:
    \begin{align*}
        V_T &= \max_{i \in [m]}V_{T,i} \\
        &\le 2\sum_{t \in [T]} \zetavec_{h(t)}^\top \vx_{h(t)} \\
        &= 2\sum_{t \in [T]} \sum_{a \in [K]} \sqrt{\frac{2\log(\nicefrac{2}{\delta'})}{\max\{\widetilde{N}_{h(t)}(a),1\}(\rho'-2\delta')^2}} \vx_t(a) \label{seq:regret_soft_four} \\
        &\le 2\sqrt{2}\sum_{t \in [T]} \sum_{a \in [K]} \sqrt{\frac{2\log(\nicefrac{2}{\delta'})}{\max\{N_t(a),1\}(\rho'-2\delta')^2}} \vx_t(a),
    \end{align*}
    concluding the proof.
\end{proof}

\begin{restatable}{lemma}{ConcentrationLemma}
    \label{lem:concentration}
    Let $\vx_1,\dots,\vx_T$ be the strategies played by any learning algorithm.
    For every $\delta_{\textnormal{A}} \in (0,1)$ it holds that:
    \begin{equation*}
        \sum_{t \in [T]} \sum_{a \in [K]} \frac{1}{\sqrt{\max\{1,N_t(a)\}}} \vx_t(a) \le 2\sqrt{KT} +\sqrt{2T\log\left(\frac{2}{\delta_{\textnormal{A}}}\right)}
    \end{equation*}
    with probability at least $\delta_{\textnormal{A}}$.
\end{restatable}
\begin{proof}    
    Let us consider the martingale difference sequence $\{Z_t\}_{t=1}^T$, where:
    \begin{align*}
        Z_t \coloneqq \sum_{a \in [K]}\frac{\vx_t(a) - \mathbb{I}\{a_t=a\}}{\sqrt{\max\{1,N_t(a)\}}} \quad \forall t \in [T].
    \end{align*}
    We observe that for every $t \in [T]$:
    \begin{align*}
        Z_t &= \sum_{a \in [K]}\frac{\vx_t(a)}{\sqrt{\max\{1,N_t(a)\}}} - \frac{1}{\sqrt{\max\{1,N_t(a_t)\}}} \\
        &\le \sum_{a \in [K]}\frac{\vx_t(a)}{\min_{a' \in [K]}\sqrt{\max\{1,N_t(a')\}}} \\
        &=\frac{1}{\min_{a' \in [K]}\sqrt{\max\{1,N_t(a')\}}} \le 1,
    \end{align*}
    and $Z_t \ge 1$. Therefore $|Z_t| \le 1$, and by applying the Azuma-Hoeffding inequality we have:
    \begin{equation}
        \label{eq:proof_concentration_azuma}
        \sum_{t \in [T]} Z_t \le \sqrt{2T\log\left(\frac{2}{\delta_{\text{A}}}\right)}
    \end{equation}
    with probability at least $\delta_{\text{A}}$.


    As a further step, we bound the quantity $\sum_{t \in [T]}\sum_{a \in [K]} \frac{\mathbb{I}\{a_t=a\}}{\sqrt{\max\{1,N_t(a)\}}}$ as follows:
    \begin{align*}
        \sum_{t \in [T]}\sum_{a \in [K]} \frac{\mathbb{I}\{a_t=a\}}{\sqrt{\max\{1,N_t(a)\}}} &=\sum_{t \in [T]}\frac{1}{\sqrt{\max\{1,N_t(a_t)\}}} \\
        &=\sum_{a \in [K]} \sum_{t : a_t=a} \frac{1}{\sqrt{\max\{1,N_t(a)\}}} \\
        &=\sum_{a \in [K]} \sum_{s=0}^{N_T(a)} \frac{1}{\sqrt{\max\{1,s\}}} \\
        &=\sum_{a \in [K]} \left(1+\sum_{s=1}^{N_T(a)} \frac{1}{\sqrt{s}} \right) \\
        &\le \sum_{a \in [K]} (1+2\sqrt{N_T(a)}-1) \\
        &\le 2\sqrt{KT},
    \end{align*}
    where the first inequality holds applying the bound $\sum_{s=1}^N \nicefrac{1}{\sqrt{s}} \le 2\sqrt{N} -1$, and the second by Cauchy-Schwartz and the fact that $\sum_{a \in [K]}N_T(a)=T-1$.
    Combining this result with Equation~\eqref{eq:proof_concentration_azuma} we have:
    \begin{align*}
        \sum_{t \in [T]} \sum_{a \in [K]} \frac{1}{\sqrt{\max\{1,N_t(a)\}}} \vx_t(a) &\le \sum_{a \in [K]} \frac{\mathbb{I}\{a_t=a\}}{\sqrt{\max\{1,N_t(a)\}}} +\sqrt{2T\log\left(\frac{2}{\delta_{\text{A}}}\right)} \\
        &\le 2 \sqrt{KT} +\sqrt{2T\log\left(\frac{2}{\delta_{\text{A}}}\right)},
    \end{align*}
    with probability at least $1-\delta_{\text{A}}$.
\end{proof}

\MainTheorem*
\begin{proof}
    According to Lemma~\ref{lem:regret_clean_event} and Lemma~\ref{lem:violations_clean_event}, with probability $1-\mathbb{P}(\Bar{\mathcal{E}})$ we have:
    \begin{equation*}
        R_T \le 2\sqrt{2}\sum_{t \in [K]} \sum_{a \in [K]} \sqrt{\frac{2\log(\nicefrac{2}{\delta'})}{\max\{1,N_t(a)\}(\rho'-2\delta')^2}} \vx_t(a).
    \end{equation*}
    and:
    \begin{equation*}
        V_T \le 2\sqrt{2}\sum_{t \in [T]} \sum_{a \in [K]} \sqrt{\frac{2\log(\nicefrac{2}{\delta'})}{\max\{1,N_t(a)\}(\rho'-2\delta')^2}} \vx_t(a).
    \end{equation*}
    Thanks to Lemma~\ref{lem:concentration}, we can bound the double summation as:
    \begin{equation*}
        \sum_{t \in [T]} \sum_{a \in [K]} \sqrt{\frac{2\log(\nicefrac{2}{\delta'})}{\max\{1,N_t(a)\}(\rho'-2\delta')^2}} \vx_t(a) \le \sqrt{\frac{ 2\log\left(\frac{2}{\delta'}\right) }{(\rho'-2\delta')^2} }\left(2\sqrt{KT} +\sqrt{2T\log\left(\frac{4}{\delta}\right)} \right)
    \end{equation*}
    with probability at least $1-\delta/2$.

    As a result, with probability at least $1-\mathbb{P}(\Bar{\mathcal{E}}) -\delta/2$:
    \begin{align*}
        R_T &\le 2\sqrt{2}\sqrt{\frac{2 \log\left(\frac{2}{\delta'} \right)}{(\rho'-2\delta')^2}} \left(2 \sqrt{KT} +\sqrt{2T\log\left(\frac{4}{\delta}\right)} \right) \\
        &\le \mathcal{O}\left( \sqrt{\frac{\log\left(\frac{1}{\delta'} \right)}{(\rho'-2\delta')^2} KT\log\left( \frac{1}{\delta} \right) } \right) \\
        &= \mathcal{O}\left( (m+1)K^2\log(T) \sqrt{\frac{\log\left(\frac{mK^2\log(T)}{\delta} \right)}{(\rho-2\delta)^2} KT\log\left( \frac{1}{\delta} \right) } \right) \\
        &= \widetilde{\mathcal{O}}\left(\frac{1}{\rho-2\delta} \log\left(\frac{1}{\delta}\right) mK^2\sqrt{KT}  \right),
    \end{align*}
    where the first equality holds as $\delta' = \Theta\left(\frac{\delta}{mK^2\log(T)}\right)$ and $\rho' = \Theta\left(\frac{\rho}{mK^2\log(T)}\right)$.
    A similar argument shows that, conditioned on the same event, the cumulative positive violation is at most:
    \begin{equation*}
        V_T \le \widetilde{\mathcal{O}}\left(\frac{1}{\rho-2\delta} \log\left(\frac{1}{\delta}\right) mK^2\sqrt{KT}  \right).
    \end{equation*}
    Therefor,e the inequalities in the statement hold with probability at least $1-\mathbb{P}(\Bar{\mathcal{E}}) -\delta/2$.
    
    In order to conclude the proof, we show that $\mathbb{P}(\Bar{\mathcal{E}}) \le \delta/2$.
    By Lemma~\ref{lem:replicable_estiamtor}, there are at most $K\lceil\log(T)\rceil$ epochs (excluding epoch $h=0$ where no estimator is computed).
    Furthermore, by construction, Algorithm~\ref{alg:main} computes $K(m+1)$ estimators for epoch, each one being within the appropriate interval with probability at least $1-\delta'.$ (Lemma~\ref{lem:replicable_estiamtor}).
    By a union bound, we have that the clean event $\mathcal{E}$ holds with probability at least $1-K^2(m+1)\lceil \log_2(T) \rceil \delta'$.
    Since $\delta' \le \frac{\delta}{2K^2(m+1)\lceil \log_2(T) \rceil}$, we have that $\mathbb{P}(\Bar{\mathcal{E}}) \le \delta/2$, concluding the proof.
\end{proof}

\section{Proofs from Section~\ref{sec:hard_constraints}}
\label{App:hard}

\ReplicableHardLemma*
\begin{proof}
    The statement can be proven by the same argument employed for Lemma~\ref{lem:replicability_and_epochs_unconstrained}.
\end{proof}

\begin{restatable}{lemma}{SigmaBounds}
    \label{lem:sigma_bounds}
    During the execution of Algorithm~\ref{alg:main_hard_constr}, it holds that $0 \le \sigma_h \le \frac{1}{1+\lambdaMin} \le 1$ for every epoch $h \in [0,H]$.
\end{restatable}
\begin{proof}
    Fix an epoch $h$.
    If $\mathcal{V}_h = \emptyset$, then $\sigma_h=0$ (Equation~\eqref{eq:sigma_h_def}) and the statement holds.
    Suppose instead that $\mathcal{V}_h \neq \emptyset$.
    There exists a constraint $j \in \mathcal{V}_h$ (\emph{i.e.}, $\min\{(\widehat{\vg}_{h,j} +\zetavec_h)^\top \widetilde{\vx}_h,1\} \ge \alpha_j)$ such that:
    \begin{align*}
        \sigma_h = \frac{y}{y +\lambda_j} \le \frac{1}{1+\lambda_j} \le \frac{1}{1+\lambdaMin} \le 1,
    \end{align*}
    where $y \coloneqq \min\{(\widehat{\vg}_{h,j} +\zetavec_h)^\top \widetilde{\vx}_h,1\} -\alpha_j \in [0,1]$.
    The chain of inequalities is verified as $y>0$ and $\lambda_j \ge \lambdaMin > 0$.
\end{proof}

\HardSafety*
\begin{proof}
    We first show that the statement holds with probability one conditioned on the event $\mathcal{E}$.
    We will then show that this event has probability at least $1-\delta$.
    
    Fix a round $t \in [T]$ and let $h=h(t)$ be its epoch.
    Algorithm~\ref{alg:main_hard_constr} plays $\vx_t = \vx^\circ_h$, where $\vx^\circ_h$ has been computed at the beginning of epoch $h$.
    Therefore, we have to show that $\vx^\circ_h$ is a safe strategy.
    We first observe that every strategy $\vx \in \XPessSafe_h$ is safe, conditioned on the event $\mathcal{E}$.
    Indeed, we have that for every constraint $i \in [m]$:
    \begin{align*}
        \vg_i^\top \vx \le (\widehat{\vg}_{h,i} +\zetavec_h)^\top \vx \le \alpha_i, 
    \end{align*}
    where the first inequality true under $\mathcal{E}$, and the second holds because $\vx \in \overline{\mathcal{X}}_h$.
    
    Notably, when $\mathcal{V}_h = \emptyset$, by definition $\widetilde{\vx}_h \in \XPessSafe_h$ (see Equation~\eqref{eq:pessimistic_safe_set_def}) and $\sigma_h = 0$ (Line~\ref{line:hard_constraints_2d}).
    Therefore, in this case Algorithm~\ref{alg:main_hard_constr} computes the safe strategy $\vx^\circ_h = \widetilde{\vx}_h \in \overline{\mathcal{X}}_h$, which contains only safe strategies under $\mathcal{E}$.

    Consider now the case $\mathcal{V}_h \neq \emptyset$.
    There exists a constraint $j \in \mathcal{V}_h$ (\emph{i.e.}, $\min\{(\widehat{\vg}_{h,j} +\zetavec_h)^\top \widetilde{\vx}_h,1\} \ge \alpha_j)$ such that:
    \begin{align*}
        \sigma_h = \frac{\min\{(\widehat{\vg}_{h,j} +\zetavec_h)^\top \widetilde{\vx}_h,1\} -\alpha_j}{\min\{(\widehat{\vg}_{h,j} +\zetavec_h)^\top \widetilde{\vx}_h,1\} -\alpha_j +\lambda_j}.
    \end{align*}

    Consider any constraint $i \in [m] \setminus \mathcal{V}_h$.
    Notice that $\widetilde{\vx}_h \in \XOptSafe_h$ satisfies the constraint $i$, as:
    \begin{align*}
        \vg_j^\top \widetilde{\vx}_h \le (\widehat{\vg}_{h,j} +\zetavec_h)^\top \widetilde{\vx}_h \le \alpha_i,
    \end{align*}
    where the first inequality holds under $\mathcal{E}$ and the second by Equation~\eqref{eq:optimistic_safe_set_def}.
    Since $\vx^\diamond$ also satisfies the constraint, by linearity the linear combination $\vx_h$ of $\vx^\diamond$ and $\widetilde{\vx}_h$ with coefficient $\sigma_h \in [0,1]$ (Lemma~\ref{lem:sigma_bounds}) satisfies it too.
    For any other constraint $i \in \mathcal{V}_h$, we observe that:
    \begin{align*}
        \vg_i^\top \vx_h &= \sigma_h\vg_i^\top \vx^\diamond +(1-\sigma_h)\vg_i^\top \widetilde{\vx}_h \\
        &\le \sigma_h(\alpha_i -\lambda_i) +(1-\sigma_h)\vg_i^\top \widetilde{\vx}_h \\
        &\le \sigma_h(\alpha_i -\lambda_i) +(1-\sigma_h)\min\{(\vg_{h,i}+\zetavec_h)^\top \widetilde{\vx}_h,1\} \\
        &=\min\{(\vg_{h,i}+\zetavec_h)^\top \widetilde{\vx}_h,1\} -\sigma_h(\min\{(\vg_{h,i}+\zetavec_h)^\top \widetilde{\vx}_h,1\} -\alpha_i +\lambda_i) \\
        &\le \min\{(\vg_{h,i}+\zetavec_h)^\top \widetilde{\vx}_h,1\} -\frac{\min\{(\vg_{h,i}+\zetavec_h)^\top \widetilde{\vx}_h,1\} -\alpha_i}{\min\{(\vg_{h,i}+\zetavec_h)^\top \widetilde{\vx}_h,1\} -\alpha_i +\lambda_i}(\min\{(\vg_{h,i}+\zetavec_h)^\top \widetilde{\vx}_h,1\} -\alpha_i +\lambda_i) \\
        &= \alpha_i,
    \end{align*}
    where the first inequality holds due to the Slater's condition, and the second holds under the event $\mathcal{E}$.
    For the last one, we use Equation~\eqref{eq:sigma_h_def} to upper bound $\sigma_h$ after observing that $\min\{(\vg_{h,i}+\zetavec_h)^\top \widetilde{\vx}_h,1\} -\alpha_i +\lambda_i \ge 0$ as $i \in \mathcal{V}_h$.
    As a result, $\vx^\circ_h$ is safe under event $\mathcal{E}$.

    Finally, we recall that there are at most $K\log_2(T)$ epochs (Lemma~\ref{lem:replicability_and_epochs_hard}).
    By Lemma~\ref{lem:replicable_estiamtor} and an union bound over the epochs and the $(m+1)K$ estimators computed every epoch, we have that $\mathcal{P}(\mathcal{E}) \ge 1-(m+1)K^2\log_2(T)\delta'$.
    Since Algorithm~\ref{alg:main_hard_constr} initializes $\delta' \le \frac{\delta}{(m+1)K^2\log_2(T)}$, the probability of event $\mathcal{E}$ is at least $1-\delta$.
\end{proof}

\begin{restatable}{lemma}{RegretZetaWidetildeHardLemma}
    \label{lem:regret_hard_one_sigma_zeta_x_concentration}
    Algorithm~\ref{alg:main_hard_constr} computes a sequence of strategies such that:
    \begin{equation*}
        \sum_{t \in [T]} (1-\sigma_{h(t)}) \zetavec_{h(t)}^\top \widetilde{\vx}_{h(t)} \le 2\sqrt{\frac{\log\left(\frac{2}{\delta'}\right)}{(\rho'-2\delta')^2}} \left( 2\sqrt{KT} +\sqrt{2T\log\left(\frac{2}{\delta'}\right)} \right)
    \end{equation*}
    with probability at least $1-\delta'$.
\end{restatable}
\begin{proof}
    Let us observe that:
    \begin{align*}
        \sum_{t=1}^T (1-\sigma_{h(t)})\zetavec_{h(t)}^\top \widetilde{\vx}_{h(t)} &= \sum_{t=1}^T (1-\sigma_{h(t)}) \sum_{a\in [K]} \sqrt{\frac{2\log\left(\frac{2}{\delta'}\right)}{\min\{\widetilde{N}_{h(t)}(a),1\}(\rho'-2\delta')^2}} \widetilde{\vx}_{h(t)}(a) \\
        &\le \sqrt{2} \sum_{t=1}^T (1-\sigma_{h(t)}) \sum_{a\in [K]} \sqrt{\frac{2\log\left(\frac{2}{\delta'}\right)}{\min\{N_t(a),1\}(\rho'-2\delta')^2}} \widetilde{\vx}_{h(t)}(a),
    \end{align*}
    where we applied Equation~\eqref{eq:zetavec_def} and the fact that $N_t(a) \le 2\widetilde{N}_{h(t)}(a)$ by construction.

    We now employ the definition of the per-round convex combination employed by Algorithm~\ref{alg:main_hard_constr} to write $\widetilde{\vx}_{h}$ as:
    \begin{equation*}
        \widetilde{\vx}_h = \frac{1}{1-\sigma_h} \vx_h -\frac{\sigma_h}{1-\sigma_h} \vx^\diamond.
    \end{equation*}
    By employing the identity above and discarding the negative terms, we get:
    \begin{align*}
        \sum_{t=1}^T (1-\sigma_{h(t)})\zetavec_{h(t)}^\top \widetilde{\vx}_{h(t)} &\le \sqrt{2} \sum_{t=1}^T \sum_{a\in [K]} \frac{(1-\sigma_{h(t)})}{(1-\sigma_{h(t)})} \sqrt{\frac{2\log\left(\frac{2}{\delta'}\right)}{\min\{N_t(a),1\}(\rho'-2\delta')^2}} \vx_{h(t)}(a) \\
        &= \sqrt{2} \sum_{t=1}^T \sum_{a\in [K]} \sqrt{\frac{2\log\left(\frac{2}{\delta'}\right)}{\min\{N_t(a),1\}(\rho'-2\delta')^2}} \vx_{h(t)}(a) \\
        &= \sqrt{2} \sum_{t=1}^T \sum_{a\in [K]} \sqrt{\frac{2\log\left(\frac{2}{\delta'}\right)}{\min\{N_t(a),1\}(\rho'-2\delta')^2}} \vx_t(a) \\
        &= 2\sqrt{\frac{\log\left(\frac{2}{\delta'}\right)}{(\rho'-2\delta')^2}} \sum_{t=1}^T \sum_{a \in [K]} \sqrt{\frac{1}{\min\{N_t(a),1\}}} \vx_t(a)
    \end{align*}
    The statement follows by employing Lemma~\ref{lem:concentration}.
\end{proof}

\begin{restatable}{lemma}{RegretOneHardLemma}
    \label{lem:regret_hard_xstar_xtilde}
    Conditioned on the event $\mathcal{E}$, Algorithm~\ref{alg:main_hard_constr} computes a sequence of strategies and coefficients such that;
    \begin{equation*}
        \sum_{t=1}^T (1-\sigma_{h(t)}) \vr^\top(\vx^\star -\widetilde{\vx}_{h(t)}) \le 
        4\sqrt{\frac{\log\left(\frac{2}{\delta'}\right)}{(\rho'-2\delta')^2}} \left( 2\sqrt{KT} +\sqrt{2T\log\left(\frac{2}{\delta'}\right)} \right),
    \end{equation*}
    with probability at least $1-\delta'$.
\end{restatable}
\begin{proof}
    The proof is similar to the one provided for Lemma~\ref{lem:regret_clean_event}.
    
    As a first step, we prove that, under the event $\mathcal{E}$, we have $(\widehat{\vr}_h +\zetavec_h)^\top \vx_t \ge \vr^\top \widetilde{\vx}^\star$ for every epoch $h \ge 0$.
    By construction, we have:
    \begin{equation*}
        \widetilde{\vx}_h \in \max_{\vx \in \XOptSafe_h} (\widehat{\vr}_h +\zetavec_h)^\top \vx,
    \end{equation*}
    where $\vx^\star \in \mathcal{X}_{h(t)}$ according to Lemma~\ref{lem:safe_set}.
    Consequently:
    \begin{equation}
        \label{eq:proof_regret_ucb_hard}
        (\widehat{\vr}_h +\zetavec_h)^\top \widetilde{\vx}_h \ge (\widehat{\vr}_h +\zetavec_h)^\top \vx^\star \ge \vr^\top \vx^\star,
    \end{equation}
    where the first inequality holds due to the optimality of $\widetilde{\vx}_h$ and the second is verified under the event $\mathcal{E}$.
    
    We can now prove the statement as follows:
    \begin{align*}
        \sum_{t=1}^T (1-\sigma_{h(t)}) \left(\vr^\top\vx^\star -\vr^\top\widetilde{\vx}_{h(t)} \right) &\le \sum_{t=1}^T (1-\sigma_{h(t)}) \left( (\widehat{\vr}_{h(t)} +\zetavec_{h(t)})^\top \widetilde{\vx}_{h(t)} -\widetilde{\vx}_{h(t)} \right) \\
        &= \sum_{t=1}^T (1-\sigma_{h(t)}) \left( (\widehat{\vr}_{h(t)} +\zetavec_{h(t)} -\vr)^\top \widetilde{\vx}_{h(t)} \right) \\
        &\le \sum_{t=1}^T (1-\sigma_{h(t)}) \left( (\widehat{\vr}_{h(t)} +\zetavec_{h(t)} -\widehat{\vr}_{h(t)} +\zetavec_{h(t)})^\top \widetilde{\vx}_{h(t)} \right) \\
        &=2\sum_{t=1}^T (1-\sigma_{h(t)}) \zetavec_{h(t)}^\top \widetilde{\vx}_{h(t)},
    \end{align*}
    where the inequalities hold due to Equation~\eqref{eq:proof_regret_ucb_hard} and the the event $\mathcal{E}$.
    Employing Lemma~\ref{lem:regret_hard_one_sigma_zeta_x_concentration}, the following result holds with probability at least $1-\delta'$:
    \begin{equation*}
        \sum_{t=1}^T (1-\sigma_{h(t)}) \le 4\sqrt{\frac{\log\left(\frac{2}{\delta'}\right)}{(\rho'-2\delta')^2}} \left( 2\sqrt{KT} +\sqrt{2T\log\left(\frac{2}{\delta'}\right)} \right),
    \end{equation*}
    concluding the proof.
\end{proof}

\begin{restatable}{lemma}{InitalRoundsHardLemma}
    \label{lem:initial_rounds_hard}
    Let:
    \begin{equation*}
        T_1 \coloneqq \frac{8K\ln\left(\frac{1}{\delta'}\right)}{\lambdaMin^2(\rho'-2\delta')^2} +\frac{8}{\lambdaMin^2}\ln\left(\frac{1}{\delta'}\right) \le \frac{9K\ln\left(\frac{1}{\delta'}\right)}{\lambdaMin^2(\rho'-2\delta')^2}.
    \end{equation*}
    With probability at least $1-\delta'$, Algorithm~\ref{alg:main_hard_constr} computes $\zetavec_{h(t)}(a) \le 1$ for every $T_1 \le t \le T$ and $a \in [k]$.
\end{restatable}
\begin{proof}
    According to Equation~\eqref{eq:zetavec_def}, we have $\zetavec_{h(t)}(a) \le 1$ when:
    \begin{equation*}
        \widetilde{N}_{h(t)}(a) \ge \frac{2\ln\left(\frac{1}{\delta'}\right)}{(\rho'-2\delta')^2},
    \end{equation*}
    which by construction ($\widetilde{N}_{h(t)}(a) \le N_t(a) \le 2\widetilde{N}_{h(t)}(a)$) is verified when $N_t(a) \ge B$, where $B$ is a quantity defined as:
    \begin{equation*}
        B \coloneqq \frac{4\ln\left(\frac{1}{\delta'}\right)}{(\rho'-2\delta')^2}.
    \end{equation*}
    Let us define $T_1$ as:
    \begin{equation*}
        T_1 \coloneqq \frac{2KB}{\lambdaMin^2} +\frac{8}{\lambdaMin^2}\ln\left(\frac{1}{\delta'}\right) = \frac{8K\ln\left(\frac{1}{\delta'}\right)}{\lambdaMin^2(\rho'-2\delta')^2} +\frac{8}{\lambdaMin^2}\ln\left(\frac{1}{\delta'}\right),
    \end{equation*}
    which we assume to be strictly less than $T$ (otherwise the statement is trivially satisfied).
    In the following, we show that for every $t \in [T_1,T]$ every action has been played at least $B$ times with probability at least $1-\delta'$, therefore proving the statement.

    
    Let $Y_t \coloneqq \mathbb{I}\{\zetavec_{h(t)}(a_t) > 1 \;\vee\; \zetavec_{h(t)}(a) \le 1 \;\forall a \in [K]\}$ and $S_{T_1} \coloneqq \sum_{t\in [T_1]} Y_t$.
    $S_{T_1}$ counts the number of rounds an action with confidence interval larger than one has been played, plus some extra rounds where every confidence interval was at most one.
    Since every action can be played at most $B$ times before its confidence interval goes below one, the statement is verified if:
    \begin{equation*}
        S_{T_1} \ge KB
    \end{equation*}
    with probability at least $1-\delta$.
    In order to do so, we decompose $S_{T_1}$ as:
    \begin{equation*}
        S_{T_1} = \sum_{t \in [T_1]} Y_t = \sum_{t \in [T_1]}\mathbb{E}[Y_t | \mathcal{F}_t] +\sum_{t \in [T_1]}Z_t,
    \end{equation*}
    where $\mathcal{F}_t$ represents the filtration up to the decision point of round $t$ and  $Z_t \coloneqq Y_t - \mathbb{E}[Y_t | \mathcal{F}_t]$ for every $t \in [T_1]$.
    We will lower bound the last two summations separately.

    To lower bound $\sum_{t \in [T_1]} \mathbb{E}[Y_t | \mathcal{F}_t]$, we observe that whenever at least an action $a$ has confidence interval $\zetavec_{h(t)}(a)$ larger than one, $\widetilde{\vx}_{h(t)}$ is a pure strategy, \emph{i.e.}, there exists an action $\widetilde{a}_{h(t)}$ such that $\widetilde{\vx}_{h(t)}(\widetilde{a}_{h(t)})=1$ and $\zetavec_{h(t)}(\widetilde{a}_{h(t)}) > 1$.
    In this case we have:
    \begin{equation*}
        \mathbb{E}[Y_t | \mathcal{F}_t] \ge \mathbb{P}(a_t = \widetilde{a}_{h(t)}) = \vx_{h(t)}(\widetilde{a}_{h(t)}) \ge (1-\sigma_{h(t)}) \widetilde{\vx}_{h(t)}(\widetilde{a}_{h(t)}) \ge \lambdaMin,
    \end{equation*}
    where the first inequality holds because $\zetavec_{h(t)}(\widetilde{a}_{h(t)}) > 1$, the second by the definition of the per-round convex combination employed by Algorithm~\ref{alg:main_hard_constr}, and the last one by Lemma~\ref{lem:sigma_bounds}.
    Instead, when every action has confidence interval at most one, we have $\mathbb{E}[Y_t | \mathcal{F}_t] = 1$ by definition of $Y_t$.
    Therefore:
    \begin{equation}
        \label{eq:proof_initial_rounds_hard_expected_yt_lb}
        \mathbb{E}[Y_t | \mathcal{F}_t] \ge \lambdaMin
    \end{equation}
    for every $t \in [T_1]$.
    
    We further observe that $Z_t$ is a martingale difference over the filtration $\mathcal{F}_t$. 
    Hence, we have that:
    \begin{equation*}
        \mathbb{P}\left(\sum_{t \in [T_1]}Z_t \ge KB -\lambdaMin T_1 \right) = \mathbb{P}\left(\sum_{t \in [T_1]}Z_t \ge -(\lambdaMin T_1 -KB) \right) \ge 1-\exp\left(-\frac{(\lambdaMin T_1 -KB)^2}{2T_1} \right) \ge 1-\delta',
    \end{equation*}
    where the first inequality is due to the Azuma–Hoeffding inequality (observe that $\lambdaMin T_1 -KB >0$), and the last one due to the definitions of $T_1$ and $B$.
    Combining this lower bound with Equation~\eqref{eq:proof_initial_rounds_hard_expected_yt_lb}, we have that:
    \begin{equation*}
        \sum_{t \in [T_1]} Y_t = \sum_{t \in [T_1]}\mathbb{E}[Y_t | \mathcal{F}_t] +\sum_{t \in [T_1]}Z_t \ge KB
    \end{equation*}
    with probability at least $1-\delta'$.
    This implies that every confidence interval is at most one after $T_1$ rounds with probability at least $1-\delta'$.
    finally, we observe that
    \begin{equation*}
        T_1 = \frac{8}{\lambdaMin^2}\ln\left(\frac{1}{\delta'}\right)\left(\frac{K}{(\rho'-2\delta')^2}+1\right) \le \frac{9K\ln\left(\frac{1}{\delta'}\right)}{\lambdaMin^2(\rho'-2\delta')^2},
    \end{equation*}
    concluding the proof.
\end{proof}

\begin{restatable}{lemma}{SigmaConcentrationAuxiliaryTwo}
    \label{lem:regret_hard_sigma_concentration_auxiliary_2}
    For every epoch $h$ such that Algorithm~\ref{alg:main_hard_constr} computes $\sigma_h \ge 1/2$, it holds that:
    \begin{equation*}
        \zetavec_h^\top \vx^\circ_h \ge \frac{1}{4}\lambdaMin^2.
    \end{equation*}
\end{restatable}
\begin{proof}
    Since $\sigma_h \neq 0$, the set $\mathcal{V}_h$ is not empty.
    Let $j \in \mathcal{V}_h$ be the constraint such that:
    \begin{equation}
        \sigma_h = \frac{\min\{ (\widehat{\vg}_{h,j} +\zetavec_h)^\top \widetilde{\vx}_h,1\}-\alpha_j}{\min\{(\widehat{\vg}_{h,j} +\zetavec_h)^\top \widetilde{\vx}_h,1\}-\alpha_j +\lambda_j} \ge \frac{1}{2}. \label{eq:proof_sigma_conc_aux_2_sigmaj}
    \end{equation}
    We show that $\zetavec_h^\top \widetilde{\vx}_h \ge \frac{\lambda_j}{2}$.
    We consider two cases.
    First, when $(\widehat{\vg}_{h,j} +\zetavec_h)^\top \widetilde{\vx}_h \le 1$, we get:
    \begin{align*}
        \alpha_j +\lambda_j &\le (\widehat{\vg}_{h,j} +\zetavec_h)^\top \widetilde{\vx}_h \\
        &=(\widehat{\vg}_{h,j} -\zetavec_h)^\top \widetilde{\vx}_h +2\zetavec_h^\top \widetilde{\vx}_h \\
        &\le \alpha_j +2\zetavec_h^\top \widetilde{\vx}_h,
    \end{align*}
    where the first inequality is obtained rearranging Equation~\eqref{eq:proof_sigma_conc_aux_2_sigmaj}, and the second inequality holds as, by construction, $\widetilde{\vx}_h \in \XOptSafe_h$ and it satisfies Equation~\ref{eq:optimistic_safe_set_def}.
    As a result, $\zetavec_h^\top \widetilde{\vx}_h \ge \frac{\lambda_j}{2}$ in this case.
    The other case is that $(\widehat{\vg}_{h,j} +\zetavec_h)^\top \widetilde{\vx}_h \ge 1$.
    This implies that:
    \begin{align*}
        1 &\le (\widehat{\vg}_{h,j} +\zetavec_h)^\top \widetilde{\vx}_h \\
        &=(\widehat{\vg}_{h,j} -\zetavec_h)^\top \widetilde{\vx}_h +2\zetavec^\top \widetilde{\vx}_h \\
        &\le \alpha_j +2\zetavec^\top \widetilde{\vx}_h \\
        &\le 1-\lambda_j +2\zetavec^\top \widetilde{\vx}_h,
    \end{align*}
    where the second inequality holds holds because $\widetilde{\vx}_h \in \XOptSafe_h$, and the second as rearranging Equation~\eqref{eq:proof_sigma_conc_aux_2_sigmaj} we get $\alpha_j \le 1-\lambda_j$.
    As a result, $\zetavec_h^\top \widetilde{\vx}_h \ge \frac{\lambda_j}{2}$ in both cases.

    To conclude the proof, we lower bound $\zetavec_h^\top \vx^\circ_h$ as follows:
    \begin{align*}
        \zetavec_h^\top \vx^\circ_h &= \sigma_h \zetavec^\top \vx^\diamond +(1-\sigma_h)\zetavec^\top \widetilde{\vx}_h \\
        &\ge (1-\sigma_h)\zetavec^\top \widetilde{\vx}_h \\
        &\ge \frac{\lambdaMin}{1+\lambdaMin} \zetavec^\top \widetilde{\vx}_h \\
        &\ge \frac{\lambdaMin^2}{2(1+\lambdaMin)} \\
        &\ge \frac{1}{4}\lambdaMin^2,
    \end{align*}
    where the second inequality holds as $\sigma_h \ge 1/(1+\lambdaMin)$ by Lemma~\ref{lem:sigma_bounds}, the third as $\zetavec_h^\top \widetilde{\vx}_h \ge \frac{\lambda_j}{2}$ and $\lambda_j \ge \lambdaMin$, and the last one as $\lambdaMin \le 1$.
\end{proof}

\begin{restatable}{lemma}{SigmaConcentrationAuxiliaryOne}
    \label{lem:regret_hard_sigma_concentration_auxiliary_1}
    Consider any execution of of Algorithm~\ref{alg:main_hard_constr} and let:
    \begin{equation*}
        \mathcal{T}_C \coloneqq \left\{t \in [T] \mid  \frac{\lambdaMin^2}{C} \le \zetavec_{h(t)}(a_t) \le 1  \right\}
    \end{equation*}
    for some $C > 1$. The cardinality of $\mathcal{T}_C$ is bounded by:
    \begin{equation*}
            |\mathcal{T}_C| \le \frac{16C^2}{\lambdaMin^4}\frac{\log\left(\frac{2}{\delta'}\right)}{(\rho'-2\delta')^2}.
    \end{equation*}
\end{restatable}
\begin{proof}
    We first observe that due to the definition of $\mathcal{T}_C$:
    \begin{equation*}
        \sum_{t \in \mathcal{T}_C} \zetavec_{h(t)}(a_t) \ge |\mathcal{T}_C| \frac{\lambdaMin^2}{C}
    \end{equation*}
    Therefore:
    \begin{align*}
        |\mathcal{T}_C| &\le \frac{C}{\lambdaMin^2} \sum_{t \in \mathcal{T}_C} \zetavec_{h(t)}(a_t) \\
        &=\frac{C}{\lambdaMin^2} \sum_{t \in \mathcal{T}_C} \sqrt{\frac{2\log\left(\frac{2}{\delta'}\right)}{\min\{\widetilde{N}_{h(t)}(a_t),1\}(\rho'-2\delta')^2}} \\
        &\le \frac{C}{\lambdaMin^2} \sqrt{2} \sum_{t \in \mathcal{T}_C} \sqrt{\frac{2\log\left(\frac{2}{\delta'}\right)}{\min\{N_t(a_t),1\}(\rho'-2\delta')^2}} \\
        &= \frac{C}{\lambdaMin^2} 2 \sum_{t \in \mathcal{T}_C} \sqrt{\frac{\log\left(\frac{2}{\delta'}\right)}{\min\{N_t(a_t),1\}(\rho'-2\delta')^2}} \\
        &\le 4\frac{C}{\lambdaMin^2} \sqrt{\frac{\log\left(\frac{2}{\delta'}\right)}{(\rho'-2\delta')^2} K|\mathcal{T}_C|},
    \end{align*}
    where the second inequality holds since $\widetilde{N}_t(a_t) \le 2N_t(a)$, and the last one due to the fact that $\sum_{t \in [\mathcal{T}_C]} \frac{1}{\min\{N_t(a_t),1\}} \le 2\sqrt{K|\mathcal{T}_C|}$ (the argument is analogous to the one provided in the proof of Lemma~\ref{lem:concentration}).
    Solving the inequality with respect to $|\mathcal{T}_C|$ yields:
    \begin{equation*}
        |\mathcal{T}_C| \le \frac{16C^2}{\lambdaMin^4}\frac{\log\left(\frac{2}{\delta'}\right)}{(\rho'-2\delta')^2},
    \end{equation*}
    concluding the proof.
\end{proof}

\begin{restatable}{lemma}{RegretTwoHardLemma}
    \label{lem:regret_hard_sigma_concentration}
    The sequence of coefficients $\{\sigma_{h(0)},\sigma_{h(1)},\dots,\sigma_{h(T)}\}$ computed by Algorithm~\ref{alg:main_hard_constr} satisfies:
    \begin{equation*}
        \sum_{t \in [T]} \sigma_{h(t)} \le \frac{8}{\lambdaMin} \sqrt{\frac{\log\left(\frac{2}{\delta'}\right)}{(\rho'-2\delta')^2}} \left( 2\sqrt{KT} +\sqrt{2T\log\left(\frac{2}{\delta'}\right)} \right) + \frac{1610K\log\left(\frac{2}{\delta'}\right)}{\lambdaMin^6(\rho'-2\delta')^2}
    \end{equation*}
    with probability at least $1-2\delta'$. 
\end{restatable}
\begin{proof}
    First, we argue that $\sigma_h$ satisfies:
    \begin{equation}
        \label{eq:proof_sigma_h_concnentration_upper}
        \sigma_h \le \frac{2}{\lambdaMin} \zetavec_h^\top \widetilde{\vx}_h
    \end{equation}
    for every epoch $h$.
    Equation~\eqref{eq:proof_sigma_h_concnentration_upper} is trivially satisfied when $\mathcal{V}_h = \emptyset$.
    Instead, when $\mathcal{V}_h \neq \emptyset$, there exists some $j \in \mathcal{V}_h$ such that:
    \begin{subequations}
    \begin{align}
        \sigma_h &= \frac{\min\{ (\widehat{\vg}_{h,j} +\zetavec_h)^\top \widetilde{\vx}_h,1\}-\alpha_j}{\min\{(\widehat{\vg}_{h,j} +\zetavec_h)^\top \widetilde{\vx}_h,1\}-\alpha_j +\lambda_j} \nonumber\\
        &\le \frac{(\widehat{\vg}_{h,j} +\zetavec_h)^\top \widetilde{\vx}_h-\alpha_j}{(\widehat{\vg}_{h,j} +\zetavec_h)^\top \widetilde{\vx}_h-\alpha_j +\lambda_j} \nonumber\\
        &\le \frac{(\widehat{\vg}_{h,j} +\zetavec_h)^\top \widetilde{\vx}_h-\alpha_j}{\lambda_j} \label{seq:sigma_ineq_1} \\
        &= \frac{(\widehat{\vg}_{h,j} -\zetavec_h)^\top \widetilde{\vx}_h +2\zetavec_h^\top \widetilde{\vx}_h -\alpha_j}{\lambda_j} \nonumber\\
        &\le \frac{\alpha_j +2\zetavec_h^\top \widetilde{\vx}_h -\alpha_j}{\lambda_j} \label{seq:sigma_ineq_2} \\
        &\le \frac{2}{\lambdaMin} \zetavec_h^\top \widetilde{\vx}_h \nonumber
    \end{align}
    \end{subequations}
    where Equation~\eqref{seq:sigma_ineq_1} holds since $(\widehat{\vg}_{h,j} +\zetavec_h)^\top \widetilde{\vx}_h-\alpha_j > 0$ for $j \in \mathcal{V}_h$, and Equation~\eqref{seq:sigma_ineq_2} as $\widetilde{\vx}_h \in \XOptSafe_h$ defined by Equation~\eqref{eq:optimistic_safe_set_def} by construction.
    Consequently, Equation~\eqref{eq:proof_sigma_h_concnentration_upper} is verified.

    We now partition the time horizon into three disjoint sets $\mathcal{T}_{\text{I}}$, $\mathcal{T}_{\text{S}}$ and $\mathcal{T}_{\text{L}}$.
    Specifically, we let $\mathcal{T}_{I} \coloneqq \{t \in [T] \mid \exists a \in [K] : \zetavec_{h(t)}(a) > 1\}$ be the set of initial rounds where the confidence intervals are large.
    Then, we let $\mathcal{T}_{\text{S}} \coloneqq \{t \in [T] \setminus \mathcal{T}_{\text{I}} \mid \sigma_{h(t)} \le 1/2\}$ and $\mathcal{T}_{\text{L}} \coloneqq [T] \setminus (\mathcal{T}_{\text{I}} \cup \mathcal{T}_{\text{S}})$.
    We will bound the sum of $\sigma_{h(t)}$ over these three set separately.

    For the first set $\mathcal{T}_{\text{I}}$, the following holds with probability at least $1-\delta'$:
    \begin{equation}
        \label{eq:proof_sum_ti_bound}
        \sum_{t \in \mathcal{T}_{\text{I}}}\sigma_{h(t)} \le |\mathcal{T}_{\text{I}}| \le \frac{9K\ln\left(\frac{1}{\delta'}\right)}{\lambdaMin^2(\rho'-2\delta')^2},
    \end{equation}
    where we exploited the fact that $\sigma_{h(t)} \le 1$ by Lemma~\ref{lem:sigma_bounds} and the upper bound on $|\mathcal{T}_{\text{I}}|$ provided by Lemma~\ref{lem:initial_rounds_hard}.

    For the second set $\mathcal{T}_{\text{S}}$, we observe the following:
    \begin{align*}
        \sum_{t \in \mathcal{T}_{\text{S}}} \sigma_{h(t)} &\le \frac{2}{\lambdaMin} \sum_{t \in \mathcal{T}_{\text{S}}} \zetavec_{h(t)}^\top \widetilde{\vx}_{h(t)} \\
        &\le \frac{4}{\lambdaMin} \sum_{t \in \mathcal{T}_{\text{S}}} (1-\sigma_{h(t)}) \zetavec_{h(t)}^\top \widetilde{\vx}_{h(t)}, \\
    \end{align*}
    where we applied Equation~\eqref{eq:proof_sigma_h_concnentration_upper} and the fact that $1-\sigma_{h(t)} \ge \frac{1}{2}$ for $t \in \mathcal{T}_{\text{S}}$. 
    As a result, by Lemma~\ref{lem:regret_hard_one_sigma_zeta_x_concentration} we can prove that:
    \begin{equation}
        \label{eq:proof_sum_ts_bound}
        \sum_{t \in \mathcal{T}_{\text{S}}} \sigma_{h(t)} \le \frac{8}{\lambdaMin} \sqrt{\frac{\log\left(\frac{2}{\delta'}\right)}{(\rho'-2\delta')^2}} \left( 2\sqrt{KT} +\sqrt{2T\log\left(\frac{2}{\delta'}\right)} \right)
    \end{equation}
    with probability at least $1 -\delta'$.

    Finally, we consider the set of rounds $\mathcal{T}_{\text{L}} = [T] \setminus (\mathcal{T}_{\text{I}} \cup \mathcal{T}_{\text{S}})$.
    As $\sigma_{h(t)} \in [0,1]$ by Lemma~\ref{lem:sigma_bounds}, we can upper bound $\sum_{t \in \mathcal{T}_{\text{L}}} \sigma_h(t)$ with $|\mathcal{T}_{\text{L}}|$.
    We thus provide an upper bound on $|\mathcal{T}_{\text{L}}|$.
    In order to do so, consider the auxiliary set:
    \begin{equation*}
        T_C \coloneqq \left\{t \in [T] \mid \frac{\lambdaMin^2}{C} \le \zetavec_{h(t)}(a_t) \le 1 \right\}
    \end{equation*}
    for some constant $C \ge 5$ to be defined later.
    Let $\mathcal{F}_t$ be the filtration generated by the information up to the decision of $\vx_t$.
    For any $t \in \mathcal{T}_{\text{L}}$, by the inverse Markov inequality we have:
    \begin{align*}
        \mathbb{P}\left(t \in \mathcal{T}_C | \; t \in \mathcal{T}_{\text{L}}, \mathcal{F}_t \right) &= \mathbb{P}\left( \zetavec_{h(t)}(a_t) \ge \frac{\lambdaMin^2}{C} | \; t \in \mathcal{T}_{\text{L}}, \mathcal{F}_t \right) \\
        &\ge \frac{\mathbb{E}[\zetavec_{h(t)}(a_t)|\mathcal{F}_t] -\frac{\lambdaMin^2}{C}}{1-\frac{\lambdaMin^2}{C}} \\
        &\ge \frac{ \frac{\lambdaMin^2}{4} -\frac{\lambdaMin^2}{C}}{1-\frac{\lambdaMin^2}{C}} \\
        &=\frac{ (C-4)\lambdaMin^2}{4(C-\lambdaMin^2)} \ge \frac{\lambdaMin^2}{4}
    \end{align*}
    where the first inequality holds by the inverse Markov inequality, the second by Lemma~\ref{lem:regret_hard_sigma_concentration_auxiliary_2}, and the third as $C \ge 5$ and $\lambdaMin \le 1$.
    Therefore, we have:
    \begin{align*}
        \mathbb{E} \left[\sum_{t \in [T]} \mathbb{I}\{t \in \mathcal{T}_C\} \mid \mathcal{T}_{\text{L}}\right] &\ge \mathbb{E} \left[\sum_{t \in \mathcal{T}_{\text{L}}} \mathbb{I}\{t \in \mathcal{T}_C\} \mid \mathcal{T}_{\text{L}}\right] \\
        &=\sum_{t \in \mathcal{T}_{\text{L}}} \mathbb{E} \left[ \mathbb{I}\{t \in \mathcal{T}_C\} \mid \mathcal{T}_{\text{L}}\right] \\
        &=\sum_{t \in \mathcal{T}_{\text{L}}} \mathbb{E} \left[\mathbb{E} \left[ \mathbb{I}\{t \in \mathcal{T}_C\} \mid \mathcal{F}_t\right] \mid \mathcal{T}_{\text{L}}\right] \\
        &=\sum_{t \in \mathcal{T}_{\text{L}}} \mathbb{E} \left[ \mathbb{P}(t \in \mathcal{T}_C | \mathcal{F}_t) \mid \mathcal{T}_{\text{L}}\right] \\
        &=\sum_{t \in \mathcal{T}_{\text{L}}} \mathbb{P}(t \in \mathcal{T}_C | \mathcal{F}_t, t \in \mathcal{T}_{\text{L}})  \\
        &\ge \sum_{t \in \mathcal{T}_{\text{L}}} \frac{\lambdaMin^2}{4} \\
        &= |\mathcal{T}_{\text{L}}|\frac{\lambdaMin^2}{4}.
    \end{align*}
    At the same time, thanks to Lemma~\ref{lem:regret_hard_sigma_concentration_auxiliary_1}, we have:
    \begin{equation*}
        \mathbb{E} \left[\sum_{t \in [T]} \mathbb{I}\{t \in \mathcal{T}_C\} \mid \mathcal{T}_{\text{L}}\right] \le \frac{16C^2}{\lambdaMin^4}\frac{\log\left(\frac{2}{\delta'}\right)}{(\rho'-2\delta')^2}. 
    \end{equation*}
    As a result:
    \begin{equation*}
        \frac{16C^2}{\lambdaMin^4}\frac{\log\left(\frac{2}{\delta'}\right)}{(\rho'-2\delta')^2} \ge |\mathcal{T}_{\text{L}}|\frac{\lambdaMin^2}{4}
    \end{equation*}
    Taking $C=5$ we get:
    \begin{equation*}
        |\mathcal{T}_{\text{L}}| \le \frac{1600}{\lambdaMin^6}\frac{\log\left(\frac{2}{\delta'}\right)}{(\rho'-2\delta')^2}.
    \end{equation*}
    Consequently:
    \begin{equation}
        \label{eq:proof_sum_tl_bound}
        \sum_{t \in \mathcal{T}_{\text{L}}} \sigma_{h(t)} \le |\mathcal{T}_{\text{L}}| \le \frac{1600}{\lambdaMin^6}\frac{\log\left(\frac{2}{\delta'}\right)}{(\rho'-2\delta')^2}.
    \end{equation}

    Finally, if we put together Equation~\ref{eq:proof_sum_ti_bound} and Equation~\eqref{eq:proof_sum_ts_bound}, which independently hold with probability at least $1-\delta'$, and Equation~\eqref{eq:proof_sum_tl_bound}, we have:
    \begin{align*}
        \sum_{t \in [T]} \sigma_{h(t)} &= \sum_{t \in \mathcal{T}_{\text{I}}} \sigma_{h(t)} + \sum_{t \in \mathcal{T}_{\text{S}}} \sigma_{h(t)} + \sum_{t \in \mathcal{T}_{\text{L}}} \sigma_{h(t)} \\
        &\le \frac{9K\ln\left(\frac{1}{\delta'}\right)}{\lambdaMin^2(\rho'-2\delta')^2} + \frac{8}{\lambdaMin} \sqrt{\frac{\log\left(\frac{2}{\delta'}\right)}{(\rho'-2\delta')^2}} \left( 2 \sqrt{KT} +\sqrt{2T\log\left(\frac{2}{\delta'}\right)} \right) + \frac{1600}{\lambdaMin^6}\frac{\log\left(\frac{2}{\delta'}\right)}{(\rho'-2\delta')^2} \\
        &\le \frac{8}{\lambdaMin} \sqrt{\frac{\log\left(\frac{2}{\delta'}\right)}{(\rho'-2\delta')^2}} \left( 2\sqrt{KT} +\sqrt{2T\log\left(\frac{2}{\delta'}\right)} \right) + \frac{1610K\log\left(\frac{2}{\delta'}\right)}{\lambdaMin^6(\rho'-2\delta')^2}
    \end{align*}
    with probability at least $1-2\delta'$.
\end{proof}

\begin{restatable}{lemma}{HardConstraintRegret}
    \label{lem:regret_hard_under_good_events}
    Conditioned on the event $\mathcal{E}$, the regret of Algorithm~\ref{alg:main_hard_constr} satisfies:
    \begin{equation*}
        R_T \le \frac{12}{\lambdaMin} \sqrt{\frac{\log\left(\frac{2}{\delta'}\right)}{(\rho'-\delta')^2}} \left( 2\sqrt{KT} +\sqrt{2T\log\left(\frac{2}{\delta_{\textnormal{A}}}\right)} \right) + \frac{1610K\log\left(\frac{2}{\delta'}\right)}{\lambdaMin^6(\rho'-2\delta')^2}
    \end{equation*}
    with probability at least $1-3\delta'$.
\end{restatable}
\begin{proof}
    We decompose the regret as follows:
    \begin{align*}
        R_T = \sum_{t=1}^T \left( \vr^\top (\vx^\star - \vx_t)\right) &= \sum_{t=1}^T \left( \vr^\top (\vx^\star - \vx^\circ_{h(t)}\right) \\
        &= \sum_{t=1}^T \left( \vr^\top (\vx^\star - \sigma_{h(t)} \vx^\diamond +(1-\sigma_{h(t)})\widetilde{\vx}_{h(t)} \right) \\
        &=\sum_{t=1}^T \left( \vr^\top (\sigma_{h(t)}\vx^\star +(1-\sigma_{h(t)})\vx^\star -\sigma_{h(t)} \vx^\diamond +(1-\sigma_{h(t)})\widetilde{\vx}_{h(t)} \right) \\
        &=\underbrace{\sum_{t=1}^T \sigma_{h(t)}\vr^\top (\vx^\star -\vx^\diamond)}_{\circled{1}} +\underbrace{\sum_{t=1}^T(1-\sigma_{h(t)})\vr^\top(\vx^\star -\widetilde{\vx}_{h(t)})}_{\circled{2}}.
    \end{align*}

    To bound $\circled{1}$, we employ Lemma~\ref{lem:regret_hard_sigma_concentration} as follows:
    \begin{align*}
        \sum_{t=1}^T \sigma_{h(t)}\vr^\top (\vx^\star -\vx^\diamond) &= \sum_{t=1}^T \sigma_{h(t)}(\vr^\top\vx^\star -\vr^\top\vx^\diamond) \\
        &\le \sum_{t=1}^T \sigma_{h(t)} \\
        &\le \frac{8}{\lambdaMin} \sqrt{\frac{\log\left(\frac{2}{\delta'}\right)}{(\rho'-2\delta')^2}} \left( 2\sqrt{KT} +\sqrt{2T\log\left(\frac{2}{\delta'}\right)} \right) + \frac{1610K\log\left(\frac{2}{\delta'}\right)}{\lambdaMin^6(\rho'-2\delta')^2},
    \end{align*}
    where the first inequality follows from the fact that $0 \le \vr^\top\vx^\diamond \le \vr^\top\vx^\star \le 1$, and the second one holds with probability $1-2\delta'$.

    To bound $\circled{2}$, we employ Lemma~\ref{lem:regret_hard_xstar_xtilde}, which states that
    \begin{align*}
        \sum_{t=1}^T(1-\sigma_{h(t)})\vr^\top(\vx^\star -\widetilde{\vx}_{h(t)}) &\le 4\sqrt{\frac{\log\left(\frac{2}{\delta'}\right)}{(\rho'-\delta')^2}} \left( 2\sqrt{KT} +\sqrt{2T\log\left(\frac{2}{\delta'}\right)} \right)
    \end{align*}
    with probability at least $1-\delta'$.

    Putting all together with an union bound:
    \begin{align*}
        R_T\le \frac{12}{\lambdaMin} \sqrt{\frac{\log\left(\frac{2}{\delta'}\right)}{(\rho'-\delta')^2}} \left( 2\sqrt{KT} +\sqrt{2T\log\left(\frac{2}{\delta'}\right)} \right) + \frac{1610K\log\left(\frac{2}{\delta'}\right)}{\lambdaMin^6(\rho'-2\delta')^2},
    \end{align*}
    with probability at least $1-3\delta'$.
\end{proof}

\RegretHardHighProb*
\begin{proof}
    By Lemma~\ref{lem:replicability_and_epochs_hard}, Algorithm~\ref{alg:main_hard_constr} has at most $H \le K \lceil \log_2(T) \rceil$ epochs (excluding epoch $h=0$, where no estimator is computed from stochastic samples).
    At every epoch $h = 1,2,\dots,H$, the Algorithm employs $\ReprMean$ to compute $(m+1)$ estimators for each action $a \in [K]$.
    According to Lemma~\ref{lem:replicable_estiamtor}, each estimator are within distance $\zetavec_h(a)$ from the estimated quantity with probability at least $1-\delta'$.
    By a union bound over the at most $K\lceil \log_2(T) \rceil$ epochs and $K(m+1)$ estimators for epoch, the probability of event $\mathcal{E}$ is at least $1-(m+1)K^2\lceil \log_2(T) \rceil\delta'$.
    As a result, by Lemma~\ref{lem:regret_hard_under_good_events}, the regret of Algorithm~\ref{alg:main_hard_constr} is:
    \begin{equation*}
        R_T \le \frac{12}{\lambdaMin} \sqrt{\frac{\log\left(\frac{2}{\delta'}\right)}{(\rho'-\delta')^2}} \left( 2\sqrt{KT} +\sqrt{2T\log\left(\frac{2}{\delta'}\right)} \right) + \frac{1610K\log\left(\frac{2}{\delta'}\right)}{\lambdaMin^6(\rho'-2\delta')^2} 
    \end{equation*}
    with probability at least $1-3\delta'-(m+1)K^2\lceil \log_2(T) \rceil\delta'$.
    We notice that: 
    \begin{equation*}
        3\delta'+(m+1)K^2\lceil \log_2(T) \rceil\delta' \le 2(m+1)K^2\lceil \log_2(T) \rceil\delta',
    \end{equation*}
    as $K^2 \ge 4$ and $\log_2(T) \ge 1$.
    Since Algorithm~\ref{alg:main_hard_constr} sets $\delta' \coloneqq \frac{\delta'}{2(m+1)K^2\lceil \log_2(T) \rceil}$ and $\rho' \coloneqq \frac{\rho'}{2(m+1)K^2\lceil \log_2(T) \rceil}$, simple computations show that its regret can be upper bounded as:
    \begin{align*}
        R_T &\le  \widetilde{\mathcal{O}}\left(\frac{mK^2}{\lambdaMin} \sqrt{\frac{\log\left(\frac{1}{\delta}\right)}{(\rho-2\delta)^2}} \left( \sqrt{KT} +\sqrt{T\log\left(\frac{1}{\delta}\right)} \right) + \frac{m^2K^5\log\left(\frac{1}{\delta}\right)\log(T)}{\lambdaMin^6(\rho-2\delta)^2} \right) \\
        &=\widetilde{\mathcal{O}}\left(\frac{1}{\lambdaMin(\rho-2\delta)} \log\left(\frac{1}{\delta}\right)mK^2\sqrt{KT} + \frac{m^2K^5}{\lambdaMin^6(\rho-2\delta)^2}\log\left(\frac{1}{\delta}\right)\log(T) \right)
    \end{align*}
    with probability at least $1-\delta$.
\end{proof}
\end{document}